\documentclass{article}
\usepackage{microtype}
\usepackage{graphicx}
\usepackage{subfigure}
\usepackage{booktabs} 
\usepackage{multibib}
\newcites{bench}{Benchmark Survey}
\usepackage{hyperref}
\usepackage{comment}

\usepackage[accepted]{icml2024}
\usepackage{amsmath}
\usepackage{amssymb}
\usepackage{mathtools}
\usepackage{amsthm}
\newcommand{\vx}{\bm{x}}
\newcommand{\vy}{\bm{y}}
\newcommand{\vtheta}{\bm{\theta}}

\usepackage{bm}
\usepackage{enumerate,tabularx} 
\usepackage{times}
\usepackage{comment}
\usepackage{bbm}
\theoremstyle{plain}
\newtheorem{theorem}{Theorem}[section]

\newtheorem{lemma}[theorem]{Lemma}
\newtheorem{corollary}[theorem]{Corollary}
\theoremstyle{definition}
\newtheorem{definition}[theorem]{Definition}

\theoremstyle{remark}

\usepackage{dirtytalk}

\icmltitlerunning{Out-of-Domain Generalization in Dynamical Systems Reconstruction}

\begin{document}

\twocolumn[
\icmltitle{Out-of-Domain Generalization in Dynamical Systems Reconstruction}

\begin{icmlauthorlist}
\icmlauthor{Niclas Göring}{zi,hd}
\icmlauthor{Florian Hess}{zi,hd}
\icmlauthor{Manuel Brenner}{zi,hd}
\icmlauthor{Zahra Monfared}{zi}
\icmlauthor{Daniel Durstewitz}{zi,hd,iwr}
\end{icmlauthorlist}

\icmlaffiliation{zi}{Department of Theoretical Neuroscience, Central Institute of Mental Health, Medical Faculty Mannheim, Heidelberg University, Mannheim, Germany}
\icmlaffiliation{hd}{Faculty of Physics and Astronomy, Heidelberg University, Heidelberg, Germany}
\icmlaffiliation{iwr}{Interdisciplinary Center for Scientific Computing, Heidelberg University, Heidelberg, Germany}

\icmlcorrespondingauthor{Niclas Göring}{niclas.goering@gmail.com}
\icmlcorrespondingauthor{Daniel Durstewitz}{Daniel.Durstewitz@zi-mannheim.de}

\icmlkeywords{Out-of-distribution generalization, dynamical systems, multistability, chaos, statistical learning theory, Recurrent Neural Networks, reservoir computing, Neural ODE, SINDy}

\vskip 0.3in
]

\printAffiliationsAndNotice{}

\begin{abstract}
In science we are interested in finding the governing equations, the dynamical rules, underlying empirical phenomena. While traditionally scientific models are derived through cycles of human insight and experimentation, recently deep learning (DL) techniques have been advanced to reconstruct dynamical systems (DS) directly from time series data. State-of-the-art dynamical systems reconstruction (DSR) methods show promise in capturing invariant and long-term properties of observed DS, but their ability to generalize to unobserved domains remains an open challenge. Yet, this is a crucial property we would expect from any viable scientific theory. In this work, we provide a formal framework that addresses generalization in DSR. We explain why and how out-of-domain (OOD) generalization (OODG) in DSR profoundly differs from OODG considered elsewhere in machine learning. We introduce mathematical notions based on topological concepts and ergodic theory to formalize the idea of learnability of a DSR model. We formally prove that black-box DL techniques, without adequate structural priors, generally will not be able to learn a generalizing DSR model. We also show this empirically, considering major classes of DSR algorithms proposed so far, and illustrate where and why they fail to generalize across the whole state space. Our study provides the first comprehensive mathematical treatment of OODG in DSR, and gives a deeper conceptual understanding of where the fundamental problems in OODG lie and how they could possibly be addressed in practice.
\end{abstract}

\section{Introduction}
The majority of complex systems we encounter in physics, biology, the social sciences, and beyond, can mathematically be described as systems of differential equations, whose behavior is the subject of dynamical systems theory (DST). Deriving accurate mathematical models of natural (or engineered) systems from observations for mechanistic insight, scientific understanding, and prediction, is the core of any scientific discipline. Recent years have seen a plethora of advances in the field of DS reconstruction (DSR) mostly based on deep learning (DL) approaches for inferring DS models directly from time series data and thus partly automatizing the scientific model building process \cite{brunton_discovering_2016, raissi_multistep_2018, vlachas_data-driven_2018, platt2021robust, brenner_tractable_2022, vlachas2022multiscale, hess_generalized_2023}. Like any good scientific theory, a proper DS model inferred from data should be able to generalize to novel domains (dynamical regimes) not observed during training. Here we develop a principled mathematical framework for out-of-domain (OOD) generalization (OODG) in DSR. We mathematically and numerically demonstrate that  
current data-driven SOTA methods for DSR hit fundamental limits regarding OODG, and provide some directions of how these could potentially be addressed. 
\begin{figure*}[!htb]
\begin{center}
\includegraphics[width=0.99\textwidth]{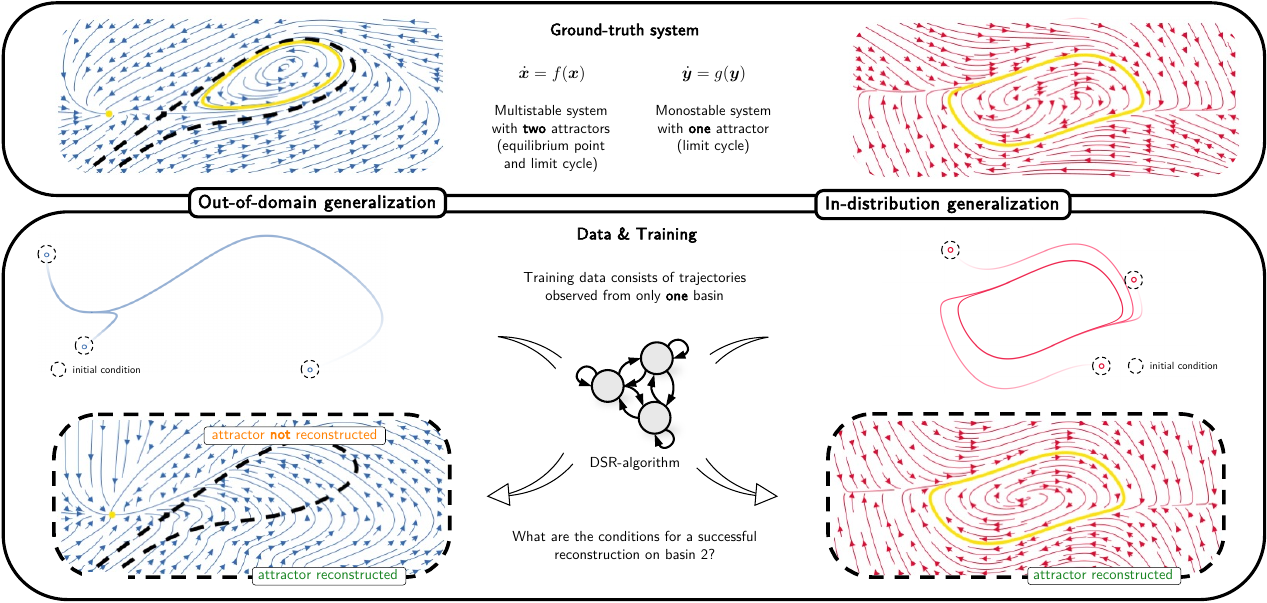}
\end{center}
\caption{In-distribution generalization within one basin (right; van-der-Pol oscillator) vs. OODG across basins (left; neuron model with a limit cycle corresponding to spiking activity and an equilibrium point corresponding to the resting potential).}
\label{fig1_true}
\normalsize
\end{figure*}
\paragraph{Current state of DSR}
Current DSR models attempt to either approximate the underlying system's vector field \cite{brunton_discovering_2016}, or try to directly learn the flow (solution) operator of the data-generating DS \cite{lu2019deeponet, vlachas_backpropagation_2020, li2020fourier, brenner_tractable_2022,hess_generalized_2023, chen2023deep}. More specifically, DSR methods have been developed based on symbolic regression \cite{brunton_discovering_2016, d2023odeformer}, on various forms of recurrent neural networks (RNNs) equipped with special training algorithms \cite{vlachas_data-driven_2018, brenner_tractable_2022,hess_generalized_2023}, on ordinary or partial differential equation (ODE/PDE)-based DL models such as Neural ODEs (N-ODE; \citet{chen2018neural,ko_homotopy-based_2023}), on operator theory \citep{lu2019deeponet, li2020fourier, chen2023deep} or on reservoir computing (RC; \citet{pathak_using_2017, verzelli2021learn, platt_systematic_2022,platt2023constraining}). State-of-the-art (SOTA) methods \cite{brenner_tractable_2022,platt2023constraining,hess_generalized_2023,jiang2023training} can generalize beyond the observed time horizon and capture an underlying system's dynamical invariants and long-term properties, like the geometry of an attractor trajectories are converging to (Fig. \ref{fig:fig_1_panel}), while providing accurate short-term forecasts on in-distribution 
test data. However, the current field of DSR has overwhelmingly focused on synthetic benchmark systems, e.g. given by low-order polynomial ODE and PDE systems, mostly in regimes where either only one (globally) attracting object exists in state space, as in the chaotic Lorenz system for common parameter settings \citep{lorenz1963deterministic}, or at least reconstruction within just one dynamical regime was sought. Only a small number of studies considered experimental data, and even less consider systems which may harbor multiple attractor objects simultaneously, so-called \textit{multistability}  (Fig. \ref{fig1_true}, left; for a detailed overview of current benchmarks in use, see Appx. \ref{appx:benchmarks}).

\paragraph{An unresolved challenge in DSR: generalization to unobserved dynamical regimes}
While current SOTA methods for DSR may generalize to nearby initial conditions close to the domain covered by the training data, \textit{which ultimately converge into the same limit set}, at current their ability to generalize to unobserved regions of state space is either not given or remains unexplored (Fig. \ref{fig1_true}). Generalization across the whole state space of the DS, or scientifically relevant portions of it, is, however, a feature any sound scientific theory should possess. The OODG problem is exacerbated in the presence of multistability, i.e., if multiple dynamical objects coexist in the same DS. In fact, even the simplest examples of low-dimensional nonlinear DS, like a damped-driven pendulum, often have multistable regimes. The problem is also of high practical relevance, as most complex DS encountered in nature and society are likely extensively multistable, with examples ranging from neuroscience \cite{schiff_controlling_1994, durstewitz_neurocomputational_2000,izhikevich_dynamical_2007, khona_attractor_2022}, optics \cite{optic1}, chemistry \cite{chemistry1}, biology \cite{biology1}, ecology \cite{ecology1}, to financial markets \cite{finance1} and climate science \cite{climate2}. For such DS, crossing the boundaries between different dynamical regimes, as induced by noise or external inputs, may lead to qualitatively completely different behavior (Fig. \ref{fig1_true}, left; Fig. \ref{fig_1_old_multistability}). 

\paragraph{Current approaches toward OODG in DSR}
Several recent studies at least partially or implicitly address the question of OODG and multistability in DSR. For instance, in \citet{ghadami_forecasting_2018, patel_using_2022, bury_predicting_2023}, the authors attempt to anticipate tipping points in non-autonomous DS and predict the post-tipping point dynamics. A related topic is the forecasting of extreme events \cite{farazmand_extreme_2018,guth_machine_2019, qi_using_2020}, which are events that are not, or only very sparsely, represented in the training data, thus constituting a form of generalization. Others consider learning DS across multiple environments defined by different parameter settings \cite{yin2022leads, bereska_continual_2022}. However, essentially all this work implicitly or explicitly assumed some observations from the domain on which generalization is sought to be available (i.e., reflecting more a form of transfer learning rather than true OODG; \citet{kirchmeyer_generalizing_2022, yin2022leads}). Another strategy to enable better generalization is to include physical domain knowledge into the model formulation or loss function, as in physics-informed neural networks (PINNs) \cite{raissi_physics-informed_2019, mehta_neural_2021, subramanian_probabilistic_2023, mouli_metaphysica_2023}, or directly ground-truth parameters of the DS studied \cite{fotiadis_disentangled_2023}. This, again, assumes we already have prior knowledge about the domain which we would like to generalize to. 

Promoting data-driven DSR models to viable scientific theories of complex systems requires a thorough understanding of whether, how, and when reconstructions generalize to the entire state space in the common empirical scenario where the measurements sample only a limited portion of that space (Fig. \ref{fig1_true}, left). This leads us to our central research question: what are the precise mathematical conditions that allow for successful reconstruction of a DS on the whole state space? 

\section{Dynamical Systems Background}
\subsection{Dynamical Systems} \label{subsec_DSconcepts}
A DS is generally comprised of a state space $M \subseteq \mathbb{R}^n$, a set of times $\mathcal{T} \subseteq \mathbb{R}$, and an evolution law. Here we focus on continuous-time systems described by ODEs 
\begin{align}\label{eq_DS}
 \dot{\bm{x}} \, = \, f(\bm{x}), \hspace{1cm} \bm{x} \in M \subseteq \mathbb{R}^n , 
\end{align}
where $f \in \mathcal{X}^1 (M)$ is a vector field (VF) from the set of functions with continuous first derivative on the (compact, metric, measurable) state space $M$. 
The VF gives rise to the evolution operator $\Phi:\mathcal{T} \times M \rightarrow M$ that maps some initial condition $\bm{x}_{0}$ to the state $\bm{x}_t$ at time $t$ 
\cite{Kuznetsov}:
\begin{equation}\label{eq:flow_operator}
\bm{x}_t = \Phi(t, \bm{x}_{0}) .
\end{equation}

\subsection{Dynamical Systems Reconstruction (DSR)} \label{subsec_DSRecons}
In data-driven DSR, the aim is to infer from time series observations a \textit{generative model} of the true underlying system, approximating either its vector field $f \in \mathcal{X}^1 (M)$ or evolution operator $\Phi(t,\bm{x})$ (hence its governing equations) given an inference algorithm from a hypothesis class $\mathcal{H}$. This goes beyond mere time series forecasting, in that we require the model to also capture dynamical invariants, that is long-term statistics and topological properties, of the underlying system. Thus, after training a DSR model should ideally be topologically conjugate to the true system and capable of producing trajectories with the same temporal and geometrical structure as those of the true system (Fig. \ref{fig:fig_1_panel}; \citet{platt_systematic_2022, platt2023constraining, hess_generalized_2023}). Note that this subsumes various more specific goals one may have in time series modeling and DS analysis. For instance, a proper DSR model should also provide excellent time series predictions, while, vice versa, a model optimized for time series prediction would not necessarily reproduce invariant statistics and the geometry and topology of a system's attractors. 

\subsection{Measure Theoretic Aspects}
\label{subsec_measure}
Measure theoretical approaches investigate the long-term statistical properties of DS, the subject of ergodic theory \cite{eckmann_ergodic_1985}. Specifically, there is a stable statistical property called the \textit{(average) occupation measure}, defined as
%-------------------
\begin{align}\label{}
    \mu_{\bm{x}_{0}, T}(B)  
     \, = \, \frac{1}{T} \int_{0}^{T} \mathbbm{1}_{B}(\bm{x}(s)) \, d s
\end{align}
%--------------------
where $\bm{x}(t)$ is a trajectory with $t \in [0, T]$, starting from $\bm{x}_0$, $B \subset \mathbb{R}^n$ is some Borel measurable set, and $\mathbbm{1}_{B} \,$ denotes an indicator function that maps elements of the set $B$ to one, and all other elements to zero. Intuitively, $ \mu_{\bm{x}_{0}, T}(B)$ measures the amount of time the trajectory $\bm{x}(t)$ spends in the set $B$.

\subsection{Topological Aspects and Attractors} 
\label{subsec_attract}
Another way to capture the long-term behavior of a DS is by studying special sets, so called invariant sets, which describe the `anatomy' of a DS. A set $U \subseteq M $ is called invariant under the flow, if $\Phi(t, U) \subseteq U \quad 
\forall t$. We can associate any point in state space with a special invariant set capturing its long-term behavior, the so-called $\omega$-limit set
\begin{align}
   \omega(\bm{x},\Phi)= \bigcap_{s \in \mathbb{R}} \overline{\{\Phi(t, \bm{x})| t>s\}}, 
\end{align}
where the overline denotes closure. If there is a set of points in state space with non-zero measure that have the same $\omega$-limit set, it is called an attractor. 
More formally: 
\begin{definition}
       An \textit{attractor} \cite{milnor_concept_1985, perko1991differential} is a closed invariant set $A \subseteq M$ such that there exists an open and forward-invariant set $B(A)= \{\bm{x} \in M| \omega(\bm{x},\Phi) \subseteq A\}$, called the \textit{basin of attraction}, with $\omega(B(A))=A$. We further require that $A$ is minimal (i.e., there is no proper subset with that same property). 
\end{definition} 
Stable equilibrium points, limit cycles and chaotic (`strange') sets are examples of attractors with increasingly complicated topology (see Fig. \ref{app_sec_attract}).  

\section{A General Framework for OODG in DSR} \label{sec_3}
\subsection{Multistability Induces Distribution Shifts}\label{sec:multistability_implies}
In statistical learning theory, out-of-sample generalization, and -- more importantly here -- OODG, is already quite well-studied (for a detailed treatment, see Appx. \ref{appx:oodg_statistical}, \citet{ben-david_theory_2010}). Generally, one assumes to have observed a set of training domains $E$ on which the data is distributed according to some domain specific distribution $p^e, \ e \in E$. The goal is to learn a function $\hat{f} \in \mathcal{F}$ in hypothesis class $\mathcal{F}$ that yields minimum risk $R^{e_{\mathrm{test}}} (\hat{f} )  :=  \mathbb{E}_{(\vx, \vy) \sim p^{e_{\mathrm{test}}}} [\ell(\hat{f} (\vx), \vy)]$ on an unseen test domain $e_{\mathrm{test}} \notin E$, where the data is distributed according to $p^{e_{\mathrm{test}}}$. Since $p^{e_{\mathrm{test}}}$ is unknown, one estimates the respective prediction error by taking the average loss across all empirically accessible domains as empirical risk:
\begin{align}
R^E_{\mathrm{emp}}(\hat{f} ) \, = \, \frac{1}{|E|}\sum_{e\in E} \frac{1}{n_e}\sum_{i=1}^{n_{e}} \ell( \hat{f} (\bm{x}^e_{i}), \bm{y}^e_{i}).    
\end{align}
\citet{yin_leads_2022} extend this definition to time-series data: 
\begin{align} \label{eq:test_risk}
    R^E_{\mathrm{emp}}(\Phi_R)= \frac{1}{|E|}\sum_{e\in E} \frac{1}{T_e}\sum_{t=1}^{T_{e}} \ell (\Phi_R(t,\vx_0^e), \bm{x}_{t}^e) ,
\end{align}
where $|E|$ is the number of domains and $T_e$ time series length.\footnote{For chaotic systems, both stationary and non-stationary probability distributions are possible \cite{parthasarathy_probability_1998}.} \citet{yin_leads_2022} mainly associate these domains with different parameter settings of the ground-truth system. For multistable DS, however, the different domains obtain a very natural interpretation.

A DS is called \textit{multistable} if it has at least two attractors coexisting in its state space (Fig. \ref{fig1_true}, left). In this study, we generally assume that multistable systems allow for a decomposition of their state space into $n$ disjoint basins of attraction \cite{milnor_concept_1985}:
\begin{align}
\label{eq_statespacedecomp}
M = \sqcup_{e=1}^n B(A_e) \sqcup \Tilde{M} \hspace{.5cm} 
\textnormal{such that} \hspace{.5cm} \mu\big(\Tilde{M}\big) = 0,
\end{align}
where $\mu$ is the Lebesgue measure. It is thus natural to define the domains in Eq. \eqref{eq:test_risk} as the different basins of attraction, each of which belongs to a different attractor. Different attractors generally give rise to different long-term dynamics with different topology (Fig. \ref{fig1_true}, left; Fig. \ref{fig_1_old_multistability}), governed by different physical measures (see Appx \ref{pg:physical_measure}). 
Hence, in each basin the trajectories follow a different statistical law (regarding their long-term evolution), implying that the challenge of reconstructing multistable DS is ultimately the same as that of understanding OODG in DSR. %\\
In monostable DS, in contrast, each trajectory in the state space is governed by the same long-term statistics (Fig. \ref{fig1_true}, right). Hence, one (sufficiently long) trajectory is already enough to specify the dynamics on the attractor. Accordingly, generalization for monostable systems essentially comes down to classical in-distribution (out-of-sample) generalization. Note that this is also true for chaotic attractors. Indeed, as illustrated in Fig. \ref{fig:fig_1_panel} (and amply demonstrated in the literature, e.g. \citet{mikhaeil_difficulty_2022, brenner_tractable_2022, hess_generalized_2023}), current SOTA methods fare very well on even short trajectories from complex chaotic systems, as long as these are monostable (or sufficient information from all basins of attraction is available, see Fig. \ref{fig_bench_lorenz_full}).
Yet, the DSR field overwhelmingly so far focused on just monostable hyperbolic attractor systems as benchmarks (Appx. \ref{appx:benchmarks}), making OODG for DSR an essentially unstudied problem.

\subsection{OODG Error}
\label{subsec_3_2}
The most commonly employed loss function $\ell$ is the mean-squared-error (MSE), derived from the maximum likelihood principle assuming i.i.d. Gaussian model residuals. While this is still the default when it comes to training RNNs, N-ODEs or RCs, it cannot be used to \textit{assess} the reconstruction for three reasons: 1) The MSE breaks down as a suitable test loss in Eq. \eqref{eq:test_risk} because of exponential trajectory divergence in chaotic DS \cite{wood_statistical_2010,koppe_identifying_2019}. Even for a perfectly reconstructed DS, numerical uncertainties in the initial conditions or small amounts of noise quickly lead to large prediction errors. This is accounted for in training methods like generalized teacher forcing \cite{hess_generalized_2023}; 2) The MSE does not capture any long-term, invariant or topological properties of the DS and its reconstruction. As discussed in Sect. \ref{subsec_measure} \& \ref{subsec_attract}, these are the central mathematical tools to study DS; 3) The MSE is not guaranteed to be sensitive to multistability, yet this property of a measure is much needed in light of Sect. \ref{sec:multistability_implies}. We thus propose a novel way of assessing generalization across state space by defining a statistical and a topological error that are provably sensitive to multistability (Theorem \ref{th_sensmulti}).

\paragraph{Statistical error}
As the MSE is not a useful quantity for comparing (chaotic) trajectories, we define the statistical error through the sliced Wasserstein-1 distance ($\textrm{SW}_1$; \citet{bonneel2015sliced}) between the occupation measures $\mu_{\bm{x}, T}^{\Phi}$ of the ground-truth DS $\Phi$ and $\mu_{\bm{x}, T}^{\Phi_R}$ of the DSR model $\Phi_R$:
\begin{align}\label{eq:swd}
\textrm{SW}_1(\mu_{\bm{x}, T}^{\Phi}, \mu_{\bm{x}, T}^{\Phi_R})
    = \mathbb{E}_{\xi \sim \mathcal{U}(\mathbb{S}^{n-1})}\left[W_1(g_{\bm{\xi}}\sharp\mu_{\bm{x}, T}^{\Phi}, g_{\bm{\xi}}\sharp\mu_{\bm{x}, T}^{\Phi_R})\right],
\end{align}
where $\mathbb{S}^{n-1} := \{\bm{\xi} \in \mathbb{R}^n \mid \lVert\bm{\xi}\rVert^2_2 = 1\}$ is the unit hyper-sphere, $g_{\bm{\xi}}\sharp\mu$ denotes the pushforward of $\mu$, $g_{\bm{\xi}}(\bm{x}) = \bm{\xi}^T\bm{x}$ is the one-dimensional slice projection, and $W_1$ the Wasserstein-1 distance \citep{villani2009optimal}. Since the expectation in Eq. \eqref{eq:swd} is intractable, it is commonly approximated by Monte Carlo sampling (see Appx. \ref{appx:dstat_details} for computational details). 
\begin{definition}
The statistical error $\mathcal{E}_{\mathrm{stat}}$ is defined as 
\begin{equation}\label{eq:dstat}
\mathcal{E}_{\mathrm{stat}}^U \big( \Phi_R  \big) := \int_{U \subseteq M}{\textrm{SW}_1(\mu_{\bm{x}, T}^{\Phi}, \mu_{\bm{x}, T}^{\Phi_R}) \ d \bm{x}},
\end{equation}
which integrates across initial conditions from a subset $U$ of state space.
\end{definition}

\paragraph{Topological error}\label{para_top}
An important concept to assess if two DS agree in their topology is \textit{topological equivalence}. Two DS are called topologically equivalent if there exists a homeomorphism between the two system's orbits preserving the direction of flow (see Appx. \ref{def_topo_equivalence} for more details). 
However, this homeomorphism is usually not known a priori and hard to access numerically. Hence, we will replace the condition of topological equivalence with three weaker conditions based on the Lyapunov spectrum, which contains topological and stability information about limit sets of a DS. Let us denote the ordered, largest $n$ Lyapunov exponents of limit sets $\omega(\vx, \Phi)$ and $\omega(\vx, \Phi_R)$ by $ \, \lambda_1 \leq  \lambda_2 \leq \cdots \leq  \lambda_n $ and $ \, \lambda^R_1 \leq  \lambda^R_2 \leq \cdots \leq  \lambda^R_n \, $, respectively. First, we require that all Lyapunov exponents agree in their signs, $sgn(\lambda_i) = sgn (\lambda^R_i) \, \forall i$. Second, we demand that the maximum Lyapunov exponent is close in relative error, $|\lambda_n - \lambda^R_n | \ / \ |\lambda_n| < \varepsilon_{\lambda_n}$, where $\varepsilon_{\lambda_n}$ is a tolerance. Lastly, the limit sets need to be close in state space, $d_H(\omega(\vx, \Phi_R ),  \, \omega(\vx, \Phi)) < \varepsilon_{d_H}$, 
assessed through the Hausdorff distance (see Appx. \ref{appx_toperror} for more details). We then define an indicator function on $M$, $\mathbbm{1}_{\Phi_R}(\vx)$, which is equal to $1$ for a given point $\vx \in U \subseteq M$ iff the associated limit sets fulfill all of the three conditions above.

\begin{definition}
We define the topological generalization error on $U \subseteq M$ as
\begin{equation}\label{eq:etop}
\mathcal{E}^U_{\text{top}}(\Phi_R) = 1 - \frac{1}{\text{vol}(U)} \int_{U \subseteq M} \mathbbm{1}_{\Phi_R}(\vx) \, d\vx.
\end{equation}
\end{definition}
In the following, we will use $\mathcal{E}_{\mathrm{gen}}$ as a placeholder for both $\mathcal{E}_{\mathrm{top}}$ and $\mathcal{E}_{\mathrm{stat}}$, and statements involving $\mathcal{E}_{\mathrm{gen}}$ must hold for \textit{both} errors.\footnote{Note that $\mathcal{E}_{\mathrm{top}}$ and $\mathcal{E}_{\mathrm{stat}}$ are solely theoretical constructs we introduce to formalize the OODG in DSR problem, not loss functions to be used in training.} These errors are highly sensitive to a failure to reconstruct multistable systems: 

\begin{theorem}\label{th_sensmulti}
Assume $\Phi$ is multistable with decomposition as in Eq. \eqref{eq_statespacedecomp} and connected basins, and there exists one attractor $A_k, \ k \leq n$,  not reconstructed by $\Phi_R$. Then, the generalization error of $\Phi_R$ is proportional to the volume of the basin of this non-reconstructed attractor:
\begin{align}
\mathcal{E}^{M_{\mathrm{test}}}_{\mathrm{gen}}(\Phi_R) \propto \operatorname{vol}(B(A_k)).
\end{align}
This statement naturally generalizes to the case of multiple non-reconstructed attractors (with different proportionality constants).
\end{theorem}
%\vspace{-.3cm}
\begin{proof}
    See Appx. \ref{app_proofthmulti}.
\end{proof}

\subsection{OOD Learnability in DSR} \label{subsec_learn}
\textit{Learnability} is a fundamental concept in statistical learning theory \citep{vapnik_nature_2000,shalev-shwartz_learnability_2010}, with many different definitions advanced \cite{valle-perez_generalization_2020}. In its simplest form, a hypothesis class is called learnable if, for any distribution of training data, the error between the learned and ground-truth function decreases monotonically with sample size and converges to zero in the limit of infinitely many data points. This concept has been extended to OODG settings in \citet{fang_is_2023}. To apply these definitions to DSR, assume the state space segregates into $n$ basins (domains), Eq. \eqref{eq_statespacedecomp}, $|E| < n$ of which form the training domains $M_{\mathrm{train}} = \cup_{e \in E} B(A_e)$ and all others the test domains $M_{\mathrm{test}}$. For simplicity, we assume we have access to the data generating process $\Phi$ on $M_{\mathrm{train}}$, such that the training data can be expressed as $\mathcal{D} \subseteq \cup_{\vx_0 \in M_{\mathrm{train}}} \Phi(T,\vx_0)$ where $[0, \ T]$ is the time interval in which the trajectories are observed. 
In line with statistical DL theory, we further assume that $\mathcal{H}$ includes hypotheses consistent with both the training and test data \cite{vallepérez2019deep, belkin_fit_2021}. In other words, there exist models within $\mathcal{H}$ that, in theory, achieve zero reconstruction error on both the training and test domains (but in practice will depend on uncertainties introduced by the DSR algorithm). Denote by $\Theta_0=\{\vtheta \in \Theta|\mathcal{E}_{\mathrm{\mathrm{gen}}}^{M_{\mathrm{train}}}(\Phi_{\theta})= 0 \}$ the set of parameters associated with models having (near) zero reconstruction error on the training domain, and by $\mathcal{H}_0$ the corresponding set of DS.   
Then, learnability in DSR boils down to:
\begin{definition}
The OODG problem $(\mathcal{H},\mathcal{D})$ defined by hypothesis class $\mathcal{H}$ and dataset $\mathcal{D}$ is \textit{strictly learnable} if
\begin{align}
    \forall \ \Phi_R \in \mathcal{H}_0: \quad \mathcal{E}_{\mathrm{\mathrm{gen}}}^{M_{\mathrm{test}}}(\Phi_R)=0.
\end{align}
\end{definition}
Hence, the OODG-problem is \textit{strictly} learnable, if zero reconstruction error on the training domain leads to zero reconstruction error on the test domain.

For highly expressive hypothesis classes there can be multiple, if not infinitely many, models in $\mathcal{H}_0$ with different generalization errors on $M_{\mathrm{test}}$. If we assume we are dealing with a \textit{parameterized} function class $\mathcal{H}_{\theta}=\{\Phi_{\theta} | \vtheta \in \Theta \subset \mathbb{R}^P \}$, as practically the case in all DL \& DSR settings, the quantity of interest becomes the distribution of generalization errors of models in $\mathcal{H}_0$: 
\begin{definition}
We define the \textit{learnability}-distribution of the OODG problem $(\mathcal{H}_{\theta},\mathcal{D})$ as
\begin{align}\label{eq_learnab}
p(\varepsilon_{\mathrm{gen}}|\mathcal{D})=\frac{1}{\mathrm{vol}(\Theta_0)}\int_{\Theta_0} \mathbbm{1}{[ \mathcal{E}^{M_{\mathrm{test}}}_{\mathrm{\mathrm{gen}}}(\Phi_{\vtheta}) = \varepsilon_{\mathrm{gen}}  ]} d \vtheta ,
\normalsize
\end{align}
\end{definition}
the probability of a model with zero reconstruction error on the training domain having a generalization error of $\varepsilon_{\mathrm{gen}}$ on $M_{\mathrm{test}}$, where $\mathbbm{1} [\cdot]$ returns $1$ if the condition in square brackets holds and $0$ otherwise. The more mass the distribution has at zero, the better the problem is learnable. In the limit of  $p(\varepsilon_{\mathrm{gen}}|\mathcal{D})$ being fully concentrated at $\varepsilon_{\mathrm{gen}}=0$, the OODG problem becomes strictly learnable. Table \ref{table_oodg_comparison} summarizes the most important differences between OODG from the perspectives of standard statistical learning theory vs. of DST as advanced here. 

\section{Results}\label{sec_4}
The learnability of an OODG problem depends both on the hypothesis class $\mathcal{H}$ as well as the chosen prior in $\mathcal{H}$ through the training algorithm. Therefore, we will examine the following two scenarios: 
\vspace{-.2cm}
\begin{itemize}
    \item \textbf{Strong prior} (Sect. \ref{sec_strong}): Library-based algorithms such as SINDy \cite{brunton_discovering_2016} introduce a strong inductive bias by explicitly providing a function class for the underlying VF.
    \vspace{-.1cm}
     \item \textbf{No prior} (Sect. \ref{sec_no}): Approaches based on universal approximators of DS, like RNNs \cite{funahashi_approximation_1993, kimura1998learning, pmlr-v120-hanson20a}, do not incorporate any explicit prior (but may still introduce implicit priors through the choice of training algorithm and parameter initialization).
\end{itemize}

\subsection{Strong Prior: Methods Based on Predefined Function Libraries}\label{sec_strong}

Following the classical statistical approach of basis expansions \cite{hastie_elements_2009, durstewitz_advanced_2017}, some popular DSR methods rest on a predefined library of basis functions in the observables \cite{brunton_discovering_2016, reinbold_using_2020}, most prominently SINDy and its further developments \cite{brunton_discovering_2016, loiseau_constrained_2018, kaiser_sparse_2018, cortiella_sparse_2021, messenger_weak_2021,kaheman_automatic_2022}. These models usually are linear in the parameters, thus easing statistical inference. Since the library of functions needs to be specified a priori, these methods induce a strong inductive bias. A strong sparsity prior on the parameters, and -- correspondingly -- sparse regression methods like LASSO or sequential thresholding \cite{brunton_discovering_2016}, ensure that only a small subset of functions from the library is selected for modeling the vector field (for details on SINDy, see Appx. \ref{appx:sindy}). More formally, this defines the class of finite-dimensional 
linearly parameterized functions with $m$ differentiable basis functions $\psi_{i}: \mathbb{R}^n \rightarrow \mathbb{R},
i=1,\dots,m$,  
  \begin{align}\label{eq_funcsp}
  \scriptsize
         \mathcal{B}_L  =\bigg\{f_{j}(\bm{x} ; \bm{\theta})=\sum_{i=1}^m \theta_{i,j}\psi_{i}(x) \ \bigg|  \ \forall j ,   \bm{\theta} \in \mathbb{R}^{m \times n} \bigg\},
    \end{align}
    \normalsize
    where these basis functions may be arbitrarily chosen. 
    In this hypothesis class, one trajectory is sufficient to fully specify the DS, unless a certain uniqueness condition is violated: 
\begin{theorem}\label{th_sc1}
Let $f \in \mathcal{B}_L$ be a multistable VF, and assume SINDy (or related) is used to learn $\Phi_R$, including the right terms from $\mathcal{B}_L$ in its library. If there exists a trajectory  $\Gamma_{\bm{x}_0} \subset \mathcal{D}$ not solving an algebraic equation in the parameters of Eq. \eqref{eq_funcsp},  
then the OODG problem given by $(\mathcal{B}_L,\mathcal{D})$ is strictly learnable.

\end{theorem}
\begin{proof}
See Appx. \ref{app_proof_sindy}.
\end{proof}
 This implies that a single trajectory from one basin is enough for the DSR model to capture the dynamics on all other basins, as long as $\mathcal{D}$ contains a trajectory not solving an algebraic equation in the parameters of Eq. \eqref{eq_funcsp}\footnote{There are in fact particular systems where \textit{each} trajectory may solve an algebraic equation, e.g. vector fields with a rational first integral like algebraic Hamiltonians (see Appx. \ref{appx_specth1} for more details).} and the correct function library is  provided\footnote{If this is not the case, sometimes SINDy may still be able to find a good approximation, depending on the degree of mismatch and the expressiveness of the library}. These observations are illustrated in Fig. \ref{fig_sindy_cycles}a. Appx. \ref{appx_para_id} provides an efficient formal procedure for checking whether the conditions on a given trajectory are met, and hence a generalizing solution could be found. In situations where the trajectory solves an algebraic equation, we can further restrict the library to find a unique solution (Corollary \ref{th_sc1_2}). We remark that these conditions are usually established by LASSO. 

\begin{figure}[h!]
  \centering \includegraphics[width=0.47\textwidth]{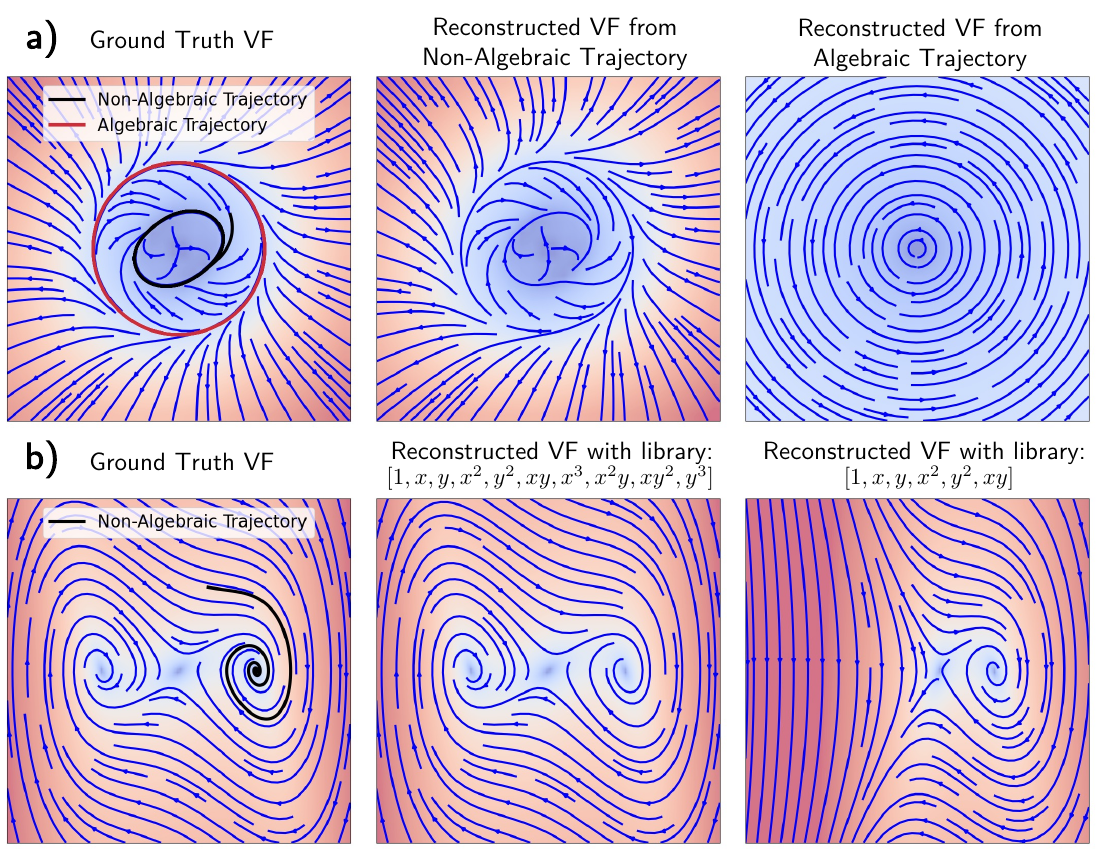}
  \caption{a) Example reconstructions using SINDy (details in Appx. \ref{appx:sindy}). The underlying VF has two cycle solutions. One solves an algebraic equation (red), while the other does not (black). The VF is only correctly identified from a trajectory containing the inner cycle (center), but not for the outer cycle (right). b) SINDy needs the proper function library to correctly infer a system across the whole state space (center). If the $3$rd order term present in the Duffing equations is lacking (right), the inferred VF may only be locally correct (or not at all for more complex systems).} \label{fig_sindy_cycles}
\end{figure}

As laid out in Appx. \ref{appx:benchmarks}, this has important implications, since many benchmark systems and scientific models in physics \cite{ramberg_description_1943}, chemistry \cite{fernandez-ramos_modeling_2006}, ecology \cite{goel_volterra_1971}, or epidemiology \cite{kermack_contribution_1927}, are expressed in terms of polynomials. However, for many complex real-world systems, like climate or the brain, which we observe through time series measurements, this assumption is likely to be violated, with polynomials at best a convenient simplification. As Figs. \ref{fig_sindy_cycles}b and \ref{fig_sindy_trigo} make clear, library approaches like SINDy will generally fail if the library does not contain the right terms describing the GT model.\footnote{In fact, SINDy fails on many empirical datasets from complex systems \cite{brenner_tractable_2022,hess_generalized_2023}.} Hence, in scientific ML we often turn to more flexible and expressive models.

\subsection{No Priors: Universal Approximators}\label{sec_no}
Next, we examine the most common data-driven approach of choosing a `black-box' model, like an RNN, to approximate the flow of the underlying system, i.e. without assuming any prior knowledge about the to-be-modeled system. If we assume that these models operate in the universal approximation limit, we can show that there is an infinity of models in the hypothesis class having zero reconstruction error on the training domain but a very high error on the whole state space or for OOD data from $M_{\mathrm{test}}$. This is in stark contrast to SINDy, where -- \textit{given the assumptions of Theorem \ref{th_sc1} are met} -- every model with zero reconstruction error on the training domain also has zero generalization error.
\begin{theorem}\label{th_1_main}
    Let $\Phi$ be a multitstable flow that is not topologically transitive (cf. Appx. \ref{def_topo_transitivity}) on $M_{\mathrm{test}}$, generated by a VF $f \in \mathcal{X}^1$. Then, the OODG problem $(\mathcal{X}^1,\mathcal{D})$ is not strictly learnable. In fact, there exists an infinite family of $f \in \mathcal{X}^1$ and an $ \varepsilon >0$ such that the corresponding flows fulfill
    \begin{align}
        \mathcal{E}^{M_{\mathrm{train}}}_{\mathrm{\mathrm{gen}}}(\Phi)=0 \quad \textnormal{ and } \quad \mathcal{E}^{M_{\mathrm{test}}}_{\mathrm{\mathrm{gen}}}(\Phi) \geq \varepsilon.
    \end{align}
\end{theorem}
\begin{proof}
See Appx. \ref{app_proof_th1}.  
\end{proof}
Note that this result is independent from the loss function used in training. Fig. \ref{fig_bench}, where data were just sampled from one basin of attraction of the multistable Duffing system (Eq. \eqref{eq:duffing_vf}), illustrates this idea for three of the most commonly used DSR models (in stark contrast to DSR performance on monostable systems, cf. Fig. \ref{fig:fig_1_panel}). We emphasize that this is not a sampling issue: Regardless of how much data are drawn from one basin, generalization fails, while increasing sample size quickly helps to identify the whole state space if data from both basins are available (Fig. \ref{fig_error_vs_sample_size}). SINDy on the other hand, \textit{provided the correct function library}, generalizes (Fig. \ref{fig_sindy_cycles}b).

\begin{figure*}[htb!]
  \centering 
  \includegraphics[width=1.0\textwidth]{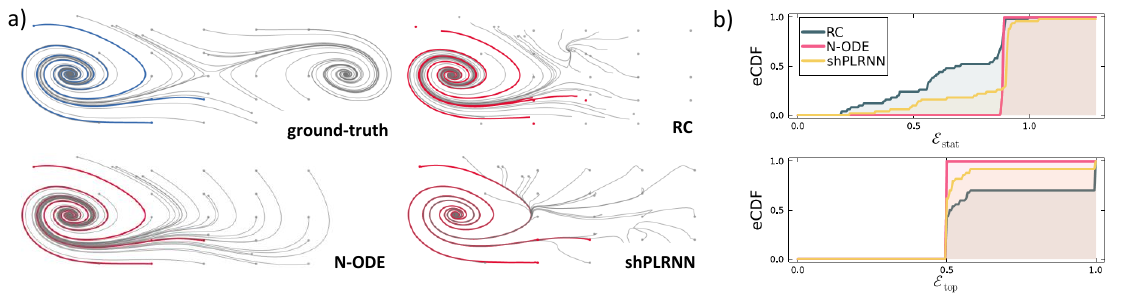 }
  \vspace{-.8cm}
  \caption{ Learnability of three SOTA DSR algorithms 
  evaluated on the Duffing system in a multistable regime. a) Reconstructions of DSR models trained on four ground-truth trajectories (blue) from one basin. Red trajectories are freely generated using initial conditions of the training data and the respective DSR model. Grey trajectories comprise example ground-truth test trajectories and generated ones from both the training basin and OOD basin. While training data trajectories align with the ground-truth, all models fail to properly generalize to the unobserved attractor/basin. b) Empirical cumulative distribution function (eCDF) of both $\mathcal{E}_{\mathrm{stat}}$ and $\mathcal{E}_{\mathrm{top}}$ based on $N=50$ independent trainings of each DSR model evaluated over a grid of initial conditions covering both basins (see Fig. \ref{fig:duffing_U_grid}).}
  \label{fig_bench}
  \normalsize
\end{figure*}
It is important to note that while on multistable settings like the one above, if trajectories are drawn from just one basin OODG will generally fail, the very same architectures can be trained to approximately zero training error on the full state space $M$ (see Fig. \ref{fig_bench_duffing_full}). This implies there are indeed regions in the loss landscape that would generalize, raising the question of why these are hardly ever discovered by the optimization algorithm. 

\subsection{Why OODG Fails}\label{OODF_failure}
We will shed light on this failure, focusing on RNNs trained with SGD. Given data $\mathcal{D}$, the probability that an RNN after training has a generalization error $\varepsilon_{\mathrm{\mathrm{gen}}}$ is formally given by
%------------------
\begin{equation}\label{eq_sgd}
 \begin{aligned}
    p_{\mathrm{SGD}}(\varepsilon_{\mathrm{stat}} \mid \mathcal{D})=\int_{\Theta} \mathbbm{1}{[ \mathcal{E}_{\mathrm{\mathrm{gen}}}^M(\Phi_{\vtheta_f}) = \varepsilon_{\mathrm{\mathrm{gen}}}  ]} \\ p_{\mathrm{opt}}\left(\vtheta_{f} \mid \vtheta_{i}, \mathcal{D}\right) p_{\mathrm{ini}}\left(\vtheta_{i}\right) d \vtheta_{i} d \vtheta_{f} ,
\end{aligned}
\end{equation}
where $p_{\mathrm{ini}}$ characterizes the initialization scheme and $p_{\mathrm{opt}}$ formalizes the training process, quantifying the probability of obtaining a final set of parameters $\vtheta_f$ given an initial set $\vtheta_i$. Under certain assumptions (cf. Appx. \ref{appx_psgd} for details), Eq. \eqref{eq_sgd} exactly aligns with the learnability distribution (Eq. \eqref{eq_learnab}). We now illustrate how the implicit biases in $p_{\mathrm{ini}}$ or $p_{\mathrm{opt}}$ will impede OODG. 

\paragraph{Simplicity bias in $p_{\mathrm{ini}}$}
In recent studies of standard NNs \cite{vallepérez2019deep,mingard2023deep} and transformers \cite{bhattamishra2023simplicity} it was shown that the parameter-function map $\mathcal{M}$ (Appx. \ref{appx_simpbias}) is biased towards `simple' functions, which in turn may explain the good generalization capability of these models on i.i.d. data (of course, time series data are not i.i.d. to begin with). Here we show that RNNs also exhibit a bias towards simplicity, which, in this case, unfortunately, manifests as a bias towards monostable DS. 

\begin{figure}[!htb]
\begin{center}
\includegraphics[width=0.42\textwidth]{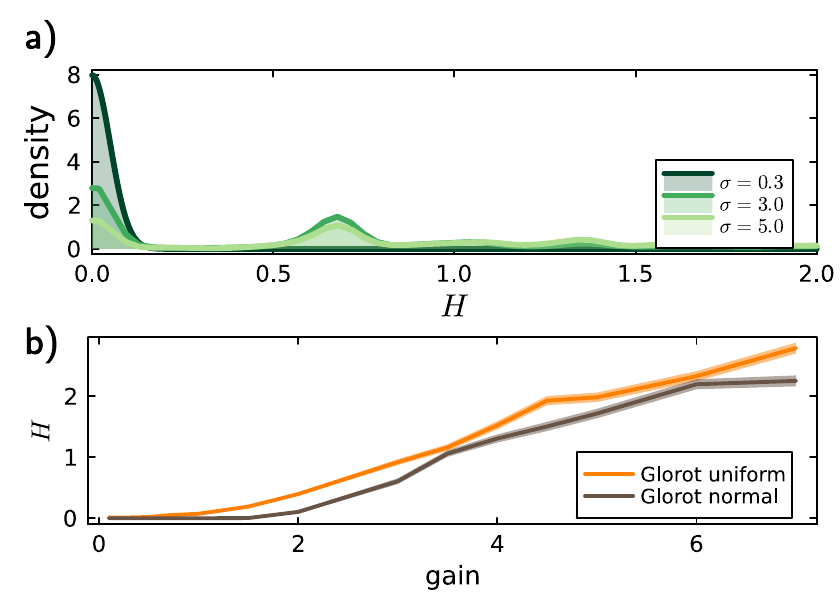}
\end{center}
 \vspace{-.4cm}
\caption{a)  Distribution of Shannon entropies (in Nat) for the limit sets of shPLRNNs ($M=2, H=100$) initialized with different gains (parameter variances) using the Glorot uniform scheme. For a low gain ($\sigma =0.3$), as predominantly used in DSR, the attractors of all models at initialization had $H=0$, which means that these consisted only of a single equilibrium point. For higher gains, further peaks at $H>0$ started to appear, implying that either more and/or higher-order objects (like cycles) exist upon initialization. b) Mean Shannon entropy for the same data plotted against gain,  using the Glorot uniform and Glorot normal initialization scheme.}
\label{fig:mean_entropy_init}
\end{figure}

To this end we initialized a shPLRNN $\Phi_{\theta}$ (\citet{hess_generalized_2023}, Appx. \ref{sec:supp:rnns}) using the Glorot uniform and Glorot normal \citep{glorot_understanding_2010} scheme where we systematically varied the gain scaling the variance. We then uniformly drew $N_I$ initial conditions and evolved them with this randomly initialized DSR model $\Phi_{\theta}$ until the resulting trajectories had 
converged to a limit set. The distribution of these limit set points across state space was then quantified through the Shannon entropy, which gives a measure for the complexity of the attractor structure at initialization (Fig. \ref{fig:mean_entropy_init}a). In Fig. \ref{fig:mean_entropy_init}b, the mean entropy is plotted as a function of gain (variance), revealing a clear trend (see also Fig. \ref{fig_exp} for a higher-dimensional RNN example). Increasing the parameter variance hence leads to more complex dynamics at initialization. However, we empirically observe that models initialized with high gains become almost impossible to train by SGD, as commonly observed for vanilla feed-forward NNs \cite{glorot_understanding_2010}. This conflict effectively biases all trainable (shPL)RNNs toward monostability, and merely increasing the gain by itself is therefore not a viable option for enhancing OODG.

\paragraph{Generalizing solutions are saddle points}
While the implicit bias introduced by $p_{\mathrm{ini}}$ plays a role in OODG failure, uncertainties in the optimization, as quantified through $p_{\mathrm{opt}}$, turned out to be even more crucial. To illustrate this, we consider the bistable Duffing oscillator (see Appx. \ref{fig:multistable_lorenz_saddles} for a chaotic multistable example) and denote by $\vtheta_{\mathrm{gen}}$ the parameters of a model generalizing across $M$, i.e., with close to zero training loss $\ell_{M}=\ell_{B(A_1)}+\ell_{B(A_2)}$ and reconstruction error $\mathcal{E}_{\mathrm{gen}}$ on both basins. We then retrain a model initialized with $\vtheta_{\mathrm{gen}}$ on trajectories from just one of the two basins, i.e. employing  $\ell_{B(A_1)}$ as a loss function.  
In Fig. \ref{fig_relearning}a we present the distribution of statistical errors for various generalizing models $\Phi_{\theta_{\mathrm{gen}}}$ and retrained models $\Phi_{\theta_{re}}$ across the two basins, $B(A_1)$ and $B(A_2)$. We observe an about $20$-fold increase in the reconstruction error of retrained compared to initialized models on $B(A_2)$, even though the error on $B(A_1)$ remained largely the same. Hence, the process of retraining effectively leads the models to \textit{unlearn} the dynamics on the second basin.
\begin{figure}[!htb]
  \centering \includegraphics[width=0.44\textwidth]{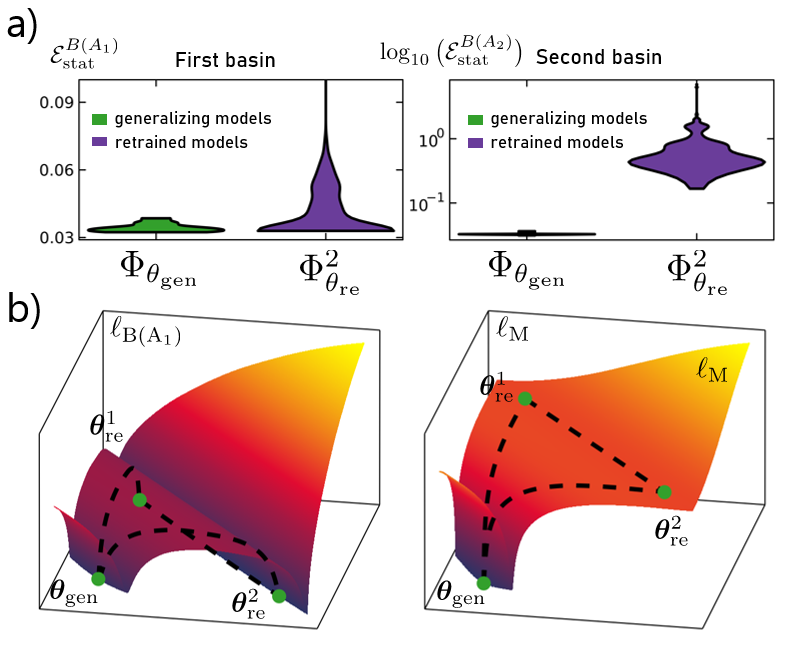}
  \caption{
  a) Statistical error distribution on basins $B(A_1)$ and $B(A_2)$ for $20$ generalizing models (green) and $20 \times 20$ models retrained (purple) using only $B(A_1)$ data. 
  b) Illustration of loss landscapes using data from just one (left) or both (right) basin(s) of attraction, with parameters corresponding to generalizing solution ($\vtheta_{\mathrm{gen}}$), 
  and to models retrained for $125k$ ($\vtheta_{re}^1$) and $250k$ ($\vtheta_{re}^2$) parameter updates, respectively.
  Note that $\ell_{\mathrm{M}}$ does not exhibit the spurious loss valley present in $\ell_{\mathrm{B(A_1)}}$.}
  \label{fig_relearning}
\end{figure}
While here we illustrated that this issue arises even in fairly simple systems like the bistable Duffing oscillator, Fig. \ref{fig:relearn_l96} shows it is equally present in higher-dimensional, more complex systems like the generalized spatially extended chaotic Lorenz-96 \citep{pelzer_finite_2020} model of atmospheric convection.

Since $\vtheta_{\mathrm{gen}}$ corresponds to a model that already agrees well with trajectories from \textit{both basins} (low $\ell_{M} \ \& \ \mathcal{E}_{\mathrm{stat}}^M$), this raises the question of why the optimizer leaves this regime during the retraining phase in the first place. To further understand this, we studied the Hessian of the loss functions evaluated on trajectories from just one ($\ell_{B(A_1)}$) or both ($\ell_{M}$) basins w.r.t. $\vtheta_{\mathrm{gen}}$ (see Tab. \ref{appx_table_eigendirections}). First, we noticed that $\vtheta_{\mathrm{gen}}$ is not a minimum but a saddle in both loss landscapes. Further, the Hessian of $\ell_{B(A_1)}$ has much fewer positive eigenvalues than that of $\ell_{M}$, implying that the saddle is more stable (with less directions to escape) when trajectories from both basins are provided. Hence, as soon as data from one basin are removed from the training set, the optimizer will run into new directions with zero or small negative eigenvalue, thus forgetting the second equilibrium point. Fig. \ref{fig_relearning}b further shows that the removal of data from the second basin leads to the emergence of spurious extrema. Current training routines may thus not be built to learn multistable systems, as they unlearn the multistable property even upon perfect initialization.
\setlength{\parskip}{0pt}

\paragraph{Generalizing minima are sharp}
In 'standard' DL, the width of minima correlates with generalization, with wider minima generalizing better than narrow ones \cite{hochreiter_flat_1997}. While certain studies have contested this correlation \cite{dinh_sharp_2017}, large-scale studies, such as \cite{jiang_fantastic_2019}, validate this association. Here, we adopt a specific 
notion of width based on the minima volumes or radii as outlined in \citet{huang_understanding_2020}, where Appx. \ref{appx_sharpmin} explains how this concept also applies to saddle regions. To further examine this idea, we trained shPLRNNs -- as above -- with identical architecture and hyperparameters once on a trajectory from just one basin and once from both basins of attraction of the Duffing system (see Fig. \ref{appx::minima_radii} for the same analysis for a chaotic multistable system). We made sure that both models have approximately the same training error when evaluated only on a single trajectory from the first basin. We then examined the width (radius) $r(\vtheta)= \|\vtheta-\vtheta_{\mathrm{min}}\|_2$ of the minimum $\vtheta_{\mathrm{min}}$ corresponding to the loss evaluated only on the one trajectory common to both (the mono- and the multistable) training setups, at a height $5 \%$ above the minimum value (other heights gave similar results, see Appx. \ref{appx_sharpmin}). On average, the minimum valleys corresponding to generalizing models, i.e. those trained on the whole state space, have a \textit{smaller} radius (Fig. \ref{fig_radius}), in contrast to the more common observation in DL that generalizing minima are usually wider. 
This, in addition to the fact that SGD is more likely to converge to wider minima \cite{chaudhari_entropy-sgd_2019, foret_sharpness-aware_2020, xie_diffusion_2021}, may further explain why generalizing minima are avoided in DSR. 

\begin{figure}[!htb]
  \centering \includegraphics[width=0.45\textwidth]{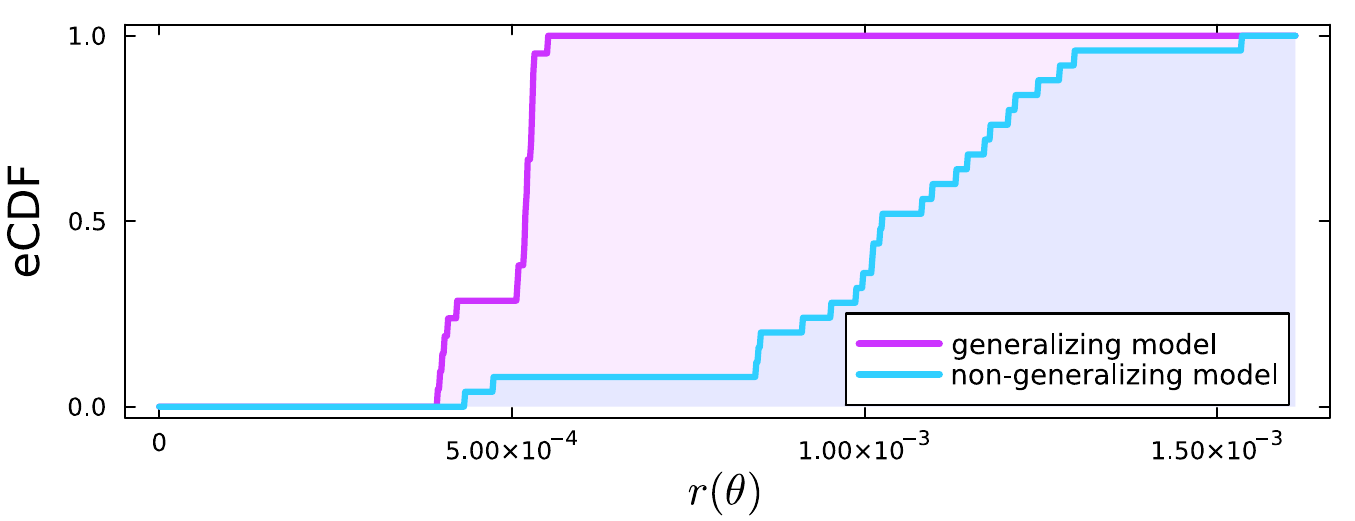}
    \vspace{-.4cm}
  \caption{eCDF of minima radii for generalizing and non-generalizing models.}
  \label{fig_radius}
\end{figure}

\section{Discussion}
Here we provide the first systematic mathematical treatment of OODG in DSR. We aimed to lay a theoretical foundation which could serve to guide the field toward future solutions of the OODG problem in DSR, by providing a new set of theoretically guided measures, providing theorems which clearly state what is, and what is not, possible, and by delineating where the hard problems lie and exactly why current SOTA algorithms struggle with them. The core problem is that most naturally observed DS will harbor many co-existing dynamical regimes, characterized by different VF topologies and long-term statistics, but usually we have observed data only from one or few of them. If we already know the correct function class, we can infer models (like SINDy) which generalize across the whole state space. But for the likely much more common empirical scenario where this is not the case, unique identification of a generalizing solution is no longer possible. In fact, if a chosen library does not even work on the training domain, this is already a strong hint that crucial terms are missing. Unfortunately, intentionally choosing a very expressive, too-large function library is not a remedy either (let alone for computational reasons), as it makes the problem underspecified. \\ 

Practically, one DS-agnostic way to potentially address OODG may be by targeting implicit biases in the initialization and, more importantly, the training processes (cf. Sect. \ref{OODF_failure}), for instance by promoting solutions that explicitly encourage multistability. Often, however, we may still need to guide the training process by a more profound physical or biological understanding of the DS in question, and evaluate trained models by explicitly (experimentally) testing novel predictions. More generally, future work may want to put the focus on training algorithms that encourage and preserve multistability and avoid overfitting the training basin.

All code used here is available at \url{https://github.com/DurstewitzLab/OODG-in-DSR}.

\section*{Acknowledgements}
This work was funded by the German Research Foundation (DFG) within Germany’s Excellence Strategy EXC 2181/1 – 390900948 (STRUCTURES) and by DFG grant Du354/15-1 to DD.

\section*{Impact Statement}
This paper presents work whose goal is to advance the field of Machine Learning. There are many potential societal consequences of our work, none which we feel must be specifically highlighted here. 

\bibliography{main}
\bibliographystyle{icml2024}

%%%%%%%%%%%%%%%%%%%%%%%%%%%%%%%%%%%%%%%%%%%%%%%%%%%%%%%%%%%%%%%%%%%%%%%%%%%%%%%
%%%%%%%%%%%%%%%%%%%%%%%%%%%%%%%%%%%%%%%%%%%%%%%%%%%%%%%%%%%%%%%%%%%%%%%%%%%%%%%
% APPENDIX
%%%%%%%%%%%%%%%%%%%%%%%%%%%%%%%%%%%%%%%%%%%%%%%%%%%%%%%%%%%%%%%%%%%%%%%%%%%%%%%
%%%%%%%%%%%%%%%%%%%%%%%%%%%%%%%%%%%%%%%%%%%%%%%%%%%%%%%%%%%%%%%%%%%%%%%%%%%%%%%
\newpage
\appendix
\onecolumn

\setcounter{figure}{0} % Restart figure numbering
\renewcommand{\thefigure}{A\arabic{figure}}% Figure counter representation
\renewcommand{\theHfigure}{A\arabic{figure}}% Hyperref figure hyperlink hook
\setcounter{table}{0}
\renewcommand{\thetable}{A\arabic{table}}

\section{Survey of Benchmark Systems}\label{appx:benchmarks}

\begin{figure*}[!htb]
\begin{center}
\includegraphics[width=0.99\textwidth]{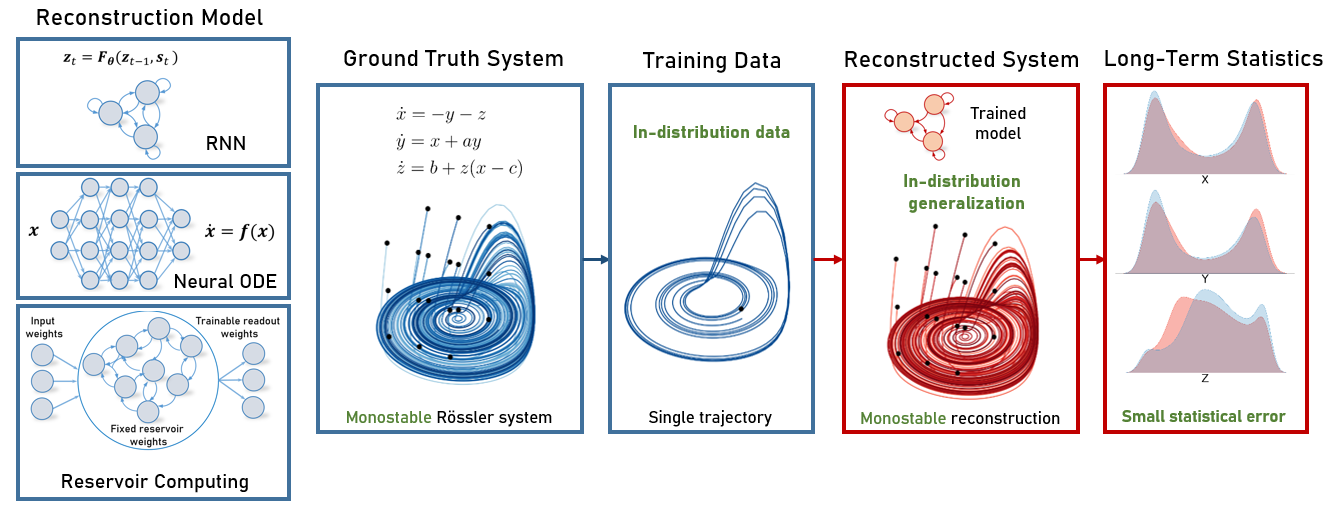}
\end{center}
\vspace{.5cm}
\caption{ In-distribution generalization in DSR.}
\label{fig:fig_1_panel}
\normalsize
\end{figure*}

We surveyed 59 papers in the field of DSR, containing a wide range of methods and applications, with respect to the benchmark systems or datasets considered (\citetbench{farmer_predicting_1987,wang_runge-kutta_1998, voss_nonlinear_2004, brunton_discovering_2016, trischler_synthesis_2016, sussillo_lfads_2016, linderman_recurrent_2016, tran_exact_2017, pathak_using_2017, raissi_multistep_2018, mohajerin_multi-step_2018}
\citetbench{vlachas_data-driven_2018, lu_attractor_2018, lusch_deep_2018, karlsson_modelling_2019, otto_linearly-recurrent_2019, raissi_physics-informed_2019, duncker_learning_2019, ayed_learning_2019, nguyen_em-like_2019, qin_data_2019, fu_dynamically_2019}
\citetbench{champion_data-driven_2019, singh_pi-lstm_2019,lee_model_2020, shalova_deep_2020, vlachas_backpropagation_2020, zhao_variational_2020, hernandez_nonlinear_2020, azencot_forecasting_2020, strauss_augmenting_2020, gilpin_deep_2020, nguyen_variational_2021}
\citebench{kraemer_unified_2021, li_data-driven_2021, lu_learning_2021, schmidt_identifying_2021, jordana_learning_2021, kim_inferring_2021, lai_structural_2021, gauthier_next_2021, goyal_learning_2021, liu_model-free_2021, schlaginhaufen_learning_2021}
\citetbench{mehta_neural_2021, zhang_learning_2022, uribarri_dynamical_2022, gilpin_chaos_2022, yin_leads_2022, rusch_long_2022, brenner_tractable_2022, lejarza_data-driven_2022, chen_time_2022, geneva_transformers_2022, mikhaeil_difficulty_2022}
\citetbench{yang_learning_2023, linot_stabilized_2023, tripura_sparse_2023, hess_generalized_2023}). This survey motivated the classification in Table \ref{tab:counts_benchmarks}, where three types of systems dominate the literature:
\begin{itemize}
    \item Simple, low-dimensional linear or nonlinear systems like the Fitz-Hugh-Nagumo equations, Lotka-Volterra system, or coupled or damped harmonic oscillators/ pendulums like the van-der-Pol oscillator. 
    \item Simple monostable 3d chaotic attractors, predominantly the Lorenz-63, Rössler or Duffing systems.
    \item Nonlinear PDEs 
    as models of fluid dynamics and convection (e.g. Burgers equation, Navier Stokes equation, Lorenz-96 or Kuramoto–Sivashinsky equations).
\end{itemize}
Experimental data or explicitly multistable systems were rarely considered (or at least not explored in their multiple stable regimes).

\begin{table}[!htb]
\caption{Classification of benchmark systems in the field of dynamical systems reconstruction.}
\centering
\begin{tabular}{l|c}
\textbf{Category} & \textbf{Counts} \\ \hline
Linear Models/Oscillators & 24 \\
Chaotic 3D Models & 29 \\
Fluid Dynamics/PDEs & 13 \\
Experimental Data & 6 \\
Multistable & 3 \\
\end{tabular}
\label{tab:counts_benchmarks}
\end{table}

\begin{figure}[!htb]
\begin{center}
\includegraphics[width=0.9\textwidth]{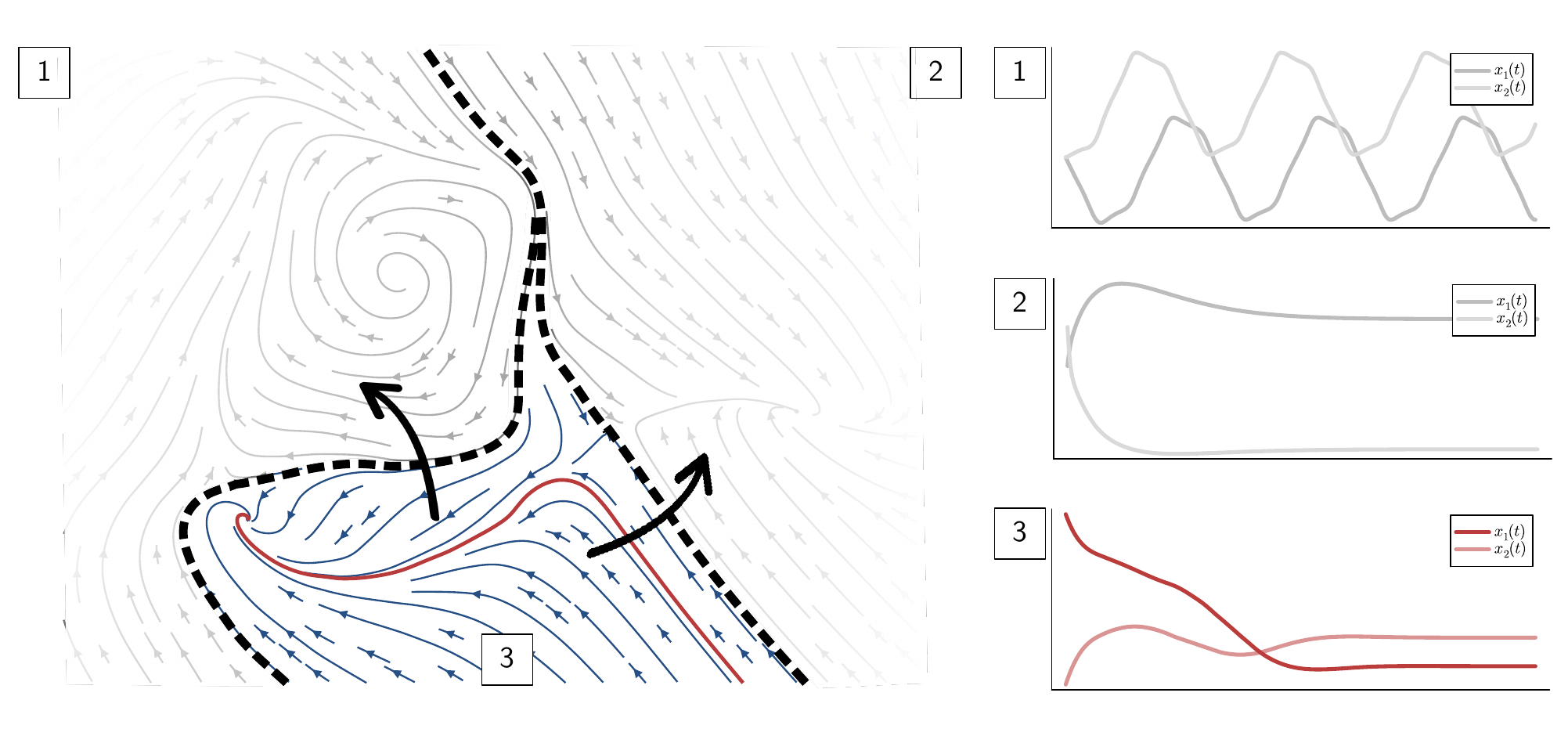}
\end{center}
\caption{Illustration of multistability. Different basins of attraction can lead to completely different dynamical regimes with different topologies.}
\label{fig_1_old_multistability}
\end{figure}

\section{Further Details on Section 2}
\subsection{Ergodic Theory and Topology}

\begin{figure}[!htb]
\begin{center}
\includegraphics[width=0.9\textwidth]{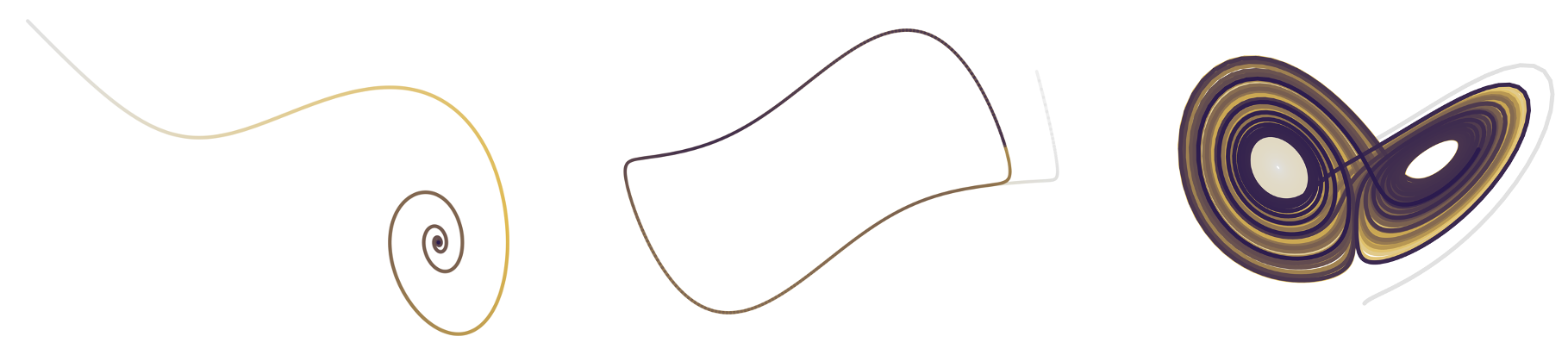}
\end{center}
\caption{Trajectories of systems with an equilibrium point (Duffing oscillator), cycle (van der Pol oscillator) and chaotic attractor (Lorenz system).}
\label{app_sec_attract}
\end{figure}

\paragraph{Physical measure}\label{pg:physical_measure}
\begin{definition}
 We call $\mu^*$ a \textit{physical measure}, related to Sinai-Ruelle-Bowen (SRB) measures \cite{climenhaga_geometric_2017}, if for some set $U$ with positive Lebesgue measure ($\mathcal{L}^n (U) > 0$) and $\bm{x}_0 \in U$, 
%------------------
\begin{align}
\lim_{T \to \infty}  \mu_{\bm{x}_{0}, T} \, = \, \mu^*.  \end{align}
%--------------------
In essence, this means that the measure is physically realisable. However, not every attractor (or even the DS itself) has to have a physical measure.
\end{definition}

\paragraph{Hausdorff Distance}
\begin{definition} \label{hausdorff_distance}
Let $X, Y$ be two non-empty subsets of a metric space $(M,d)$. The Hausdorff-distance is defined by
\begin{align}
    d_{\mathrm{H}}(X, Y)=\max \left\{\sup _{x \in X} d(x, Y), \sup _{y \in Y} d(X, y)\right\}
\end{align}
where $d(a,B)= \inf_{b \in B} d(a,b)$ with $a \in X$ and $B \subseteq X$.
\end{definition}

The choice of Hausdorff distance in this context is motivated by its robustness and sensitivity to outliers between the two sets, which in the context of the topological error makes it a suitable choice for assessing the closeness of the limit sets.

\paragraph{Topological Equivalence}

\begin{definition}\label{def_topo_equivalence}
    Let $F_1,F_2 \in \mathcal{C}^1(U)$ with flow maps $\phi_t^{F_1}, \phi^{F_2}_t$. The two vector fields (VFs) are topologically equivalent \cite{perko1991differential}, denoted by $F_1 \simeq F_2 \, $, if there exists a homeomorphism $h: U \rightarrow U$ mapping orbits of the first system onto orbits of the second system, i.e.
    \begin{align}
        \forall t \in \mathbb{R}, \ \forall x \in U: \quad \phi^{F_1}_t(x)=h^{-1} \circ \phi^{F_2}_{\tau} \circ h(x),
    \end{align}
    with $ \tau: U \times \mathbb{R} \rightarrow \mathbb{R}, \quad \frac{\partial \tau(x,t)}{\partial t} >0 \quad \forall x \in U $. This means the time direction of the orbits is preserved.
\end{definition}

Loosely speaking, two VFs are topologically equivalent if we can continuously deform one VF into the other, i.e. such that each orbit is only deformed in a continuous manner without `ripping it apart'. This implies that equilibrium points are mapped onto equilibrium points, and closed orbits to closed orbits. An open orbit will not be closed through $h$, and vice versa. 

\paragraph{Topological Transitivity}
\begin{definition}\label{def_topo_transitivity}
    Let $F \in \mathcal{C}^1(U)$ be a vector field on a topological space $U$, with its associated flow map $\phi_t^{F}$. The dynamical system induced by $F$ is said to be topologically transitive if for any two non-empty open sets $A, B \subseteq U$, there exists a time $t \in \mathbb{R}$ such that the flow map at time $t$, $\phi_t^{F}$, maps some part of $A$ into $B$; that is, $\phi_t^{F}(A) \cap B \neq \emptyset$.
\end{definition}
This implies that trajectories of the vector field $F$, starting from an arbitrary region in the space $U$, will eventually enter any other region, given that these regions are open and non-empty.

\section{Further Details on Section 3}
\subsection{OODG in Statistical Learning Theory} \label{appx:oodg_statistical}
Consider a regression or classification setting, where 
$X \subseteq \mathbb{R}^d$ and $Y \subseteq \mathbb{R}^k$ are the input and output spaces, and $\mathcal{S} = \{ (\vx_i, \vy_i)\in X \times Y \}_{i=1}^{n}$ denotes a dataset sampled from a distribution $p(\vx, \vy)$ defined on the domain $X \times Y$. 
Further assume there is a set $E$ of \textit{training} domains with cardinality $|E|$. Let the dataset of size $n_e$ from a single environment $e \in E$ be $\mathcal{S}^e = \{ 
(\vx^e_i, \vy^e_i)\in X^e \times Y^e \}_{i=1}^{n_e}$, where samples are drawn i.i.d. from the unknown, data-generating distribution $p^e(\vx, \vy)$. Consider the %full 
class $\mathcal{F}$ of functions $f: X \rightarrow Y$ and a loss function $\ell: Y \times Y \rightarrow \mathbb{R}^+ \cup \{0 \}$ measuring the goodness of fit. In OODG, the goal is to learn a generalizing, predictive function $\hat{f} \in \mathcal{F}$ from the $|E|$ training domains to obtain a minimum prediction error on an unseen test domain $e_{\mathrm{test}}$ with distribution $p^{e_{\mathrm{test}}}(\vx, \vy)$, i.e.
%---------------------
\begin{align}
\underset{\hat{f} \in \mathcal{F}}{\text{min}} \, \, \mathbb{E}_{(\vx, \vy) \sim p^{e_{\mathrm{test}}}}   \left[\ell(\hat{f} (\vx), \vy)\right],  
\end{align}
%-----------
where
%------------------
\begin{align}
 R^{e_\mathrm{test}} (\hat{f} ) \, := \, \mathbb{E}_{(\vx, \vy) \sim p^{e_{\mathrm{test}}}} \left[\ell(\hat{f} (\vx), \vy)\right] \, = \, \int  \ell(\hat{f} (\vx), \vy) \, d p^{e_\mathrm{test}}(\vx, \vy)
\end{align}
%----------------
is the expected loss for $\hat{f} $, called the test risk.
However, we cannot compute $R^{e_\mathrm{test}} (\hat{f} )$ because the distribution $p^{e_\mathrm{test}}(\vx, \vy)$ is unknown. Hence, we estimate the expectation by the sample mean across all the training domains, called \textit{empirical risk}, defined as 
%--------------
\begin{align} \label{appx:eq:test_error}
R^E_{\mathrm{emp}}(\hat{f} ) \, = \, \frac{1}{|E|}\sum_{e \in E} \frac{1}{n_e}\sum_{i=1}^{n_{e}} \ell( \hat{f} (\bm{x}^e_{i}), \bm{y}^e_{i}).    
\end{align}
%-----------
Based on this, the \textit{OODG error} is defined as the gap between the test risk and the empirical risk, 
%------
\begin{align}
\left\lvert R^{e_\mathrm{test}}(\hat{f}) - R^E_{\mathrm{emp}}(\hat{f}) \right\rvert.
\end{align}
\subsection{Statistical Error}\label{appx:dstat_details}
To compute Eq. \eqref{eq:swd}, we use a common
Monte-Carlo approximation of the expectation 
\begin{equation}\label{eq:swd_mc}
    \textrm{SW}_1(\mu_{\bm{x}, T}^{\Phi}, \ \mu_{\bm{x}, T}^{\Phi_R}) \approx \frac{1}{L} \sum_{l=1}^L W_1(g_{\bm{\xi}^{(l)}}\sharp\mu_{\bm{x}, T}^{\Phi}, \ g_{\bm{\xi}^{(l)}}\sharp\mu_{\bm{x}, T}^{\Phi_R}),
\end{equation}
where projection vectors $\bm{\xi}^{(l)} \sim \mathcal{U}(\mathbb{S}^{n-1})$ are drawn uniformly across the unit hypersphere embedded in $\mathbb{R}^n$. The Wasserstein-1 distance is computed across trajectories (empirical distributions) of the ground-truth flow $\Phi$ and the reconstructed flow $\Phi_R$. Trajectories are drawn by evolving the respective system for $T$ time units from initial conditions $\bm{x} \in \mathbb{R}^n$. Between two one-dimensional distributions, the Wasserstein-1 distance can then efficiently be computed as 
\begin{equation}\label{eq:w1_distance}
    W_1(\mu, \nu) = \int_0^1 \left\lvert F_{\mu}^{-1}(q) - F_{\nu}^{-1}(q) \right\rvert dq,
\end{equation}
where $F_\bullet^{-1}$ denotes the quantile function (inverse CDF). In practice, we 
approximate the integral in Eq. \eqref{eq:w1_distance} by evaluating the quantile functions at a resolution of $\Delta q = 10^{-3}$. 
We use $L=1000$ samples in Eq. \eqref{eq:swd_mc}. For the final error $\mathcal{E}_{stat}^U$, Eq. \eqref{eq:dstat}, we sample $K$ initial conditions from a uniformly spaced grid $\textrm{Gr}(U)=\{\vx^{(1)}, \dots, \vx^{(K)}\}$ over $U \subset M \subset \mathbb{R}^n$. The integral in Eq. \eqref{eq:dstat} is then approximated by
\begin{equation}\label{eq:eps_stat_mc}
    \mathcal{E}_{stat}^U \big( \Phi_R  \big) \approx \frac{1}{K} \sum_{x \in \textrm{Gr}(U)} \textrm{SW}_1(\mu_{\bm{x}, T}^{\Phi}, \mu_{\bm{x}, T}^{\Phi_R}).
\end{equation}

\subsection{Topological Error}\label{appx_toperror}
For the topological error, the Lyapunov spectra of orbits in limit sets $\omega(\vx, \Phi)$ and $\omega(\vx, \Phi_R)$ need to be computed. To compute the Lyapunov spectrum of continuous-time systems, i.e. ground-truth systems and Neural ODEs, we use the Julia library \texttt{TaylorIntegration.jl} \citep{perez_hernandez_2019_2562353}. For RNNs and RCs we use our own implementation of an algorithm described in \cite{geist1990comparison, vogt_lyapunov_2022}, which computes the Lyapunov spectrum by evaluating the Jacobian product along orbits of length $T$:
\begin{equation}\label{eq:lyap_spectrum}
    \lambda_i = \lim_{T \rightarrow \infty} \frac{1}{T}\log \sigma_i\left(\prod_{t=0}^{T-1} \bm{J}_{T-t}\right),
\end{equation}
where $\sigma_i$ is the $i$-th singular value. For numerical stability, the product of Jacobians is repeatedly re-orthogonalized using a QR decomposition. To ensure convergence to the limit set spectrum, transients are discarded from the computation of Eq. \eqref{eq:lyap_spectrum}. For the Duffing system (Appx. \ref{appx:duffing}), we discard the first $T_{trans} = 3000$ time steps and compute the Lyapunov spectrum across an additional $T=3000$ time steps, while re-orthogonalizing every $50$ time steps. For the multistable Lorenz-like system (Eq. \eqref{eq:multistable_lorenz}), we use $T_{trans} = 5000$ and $T=10,000$. For the tolerance of the relative error between the maxmimum Lyapunov exponents $\lambda_n$ and $\lambda_n^R$ of the ground-truth and reconstructed system, respectively, we choose $\varepsilon_{\lambda_n} = 0.25$. 
For evaluating the agreement of limit sets, $d_H(\omega(\vx, \Phi_R ),  \, \omega(\vx, \Phi)) < \varepsilon_{d_H}$, we used the same setup as for computation of the Lyapunov spectra, but only use the $T'=500$ and $T'=5000$ last time steps. 
We set $\varepsilon_{d_H} = V / L$, where $V$ is the volume of $U$, and $L$ the number of initial conditions contained in the grid $\textrm{Gr}(U)$, which is the same as used for computation of $\mathcal{E}_{stat}^U \big( \Phi_R  \big)$. For the Duffing system, this comes down to $\varepsilon_{d_H} = 40.0 / 100 = 0.4$. The integral in Eq. \eqref{eq:etop} is approximated and computed across the very same grid of initial conditions $Gr(U)$ as used for the statistical error $\mathcal{E}_{stat}^U$:
\begin{equation}
\mathcal{E}^U_{\text{top}}(\Phi_R) \approx 1 - \frac{1}{K} \sum_{\vx\in \textrm{Gr}(U)} \mathbbm{1}_{\Phi_R}(\vx)
\end{equation}
See Fig. \ref{fig:duffing_U_grid} for a visualization of the grid of initial conditions used for the Duffing system.

\begin{figure}
    \centering
    \includegraphics[width=0.6\textwidth]{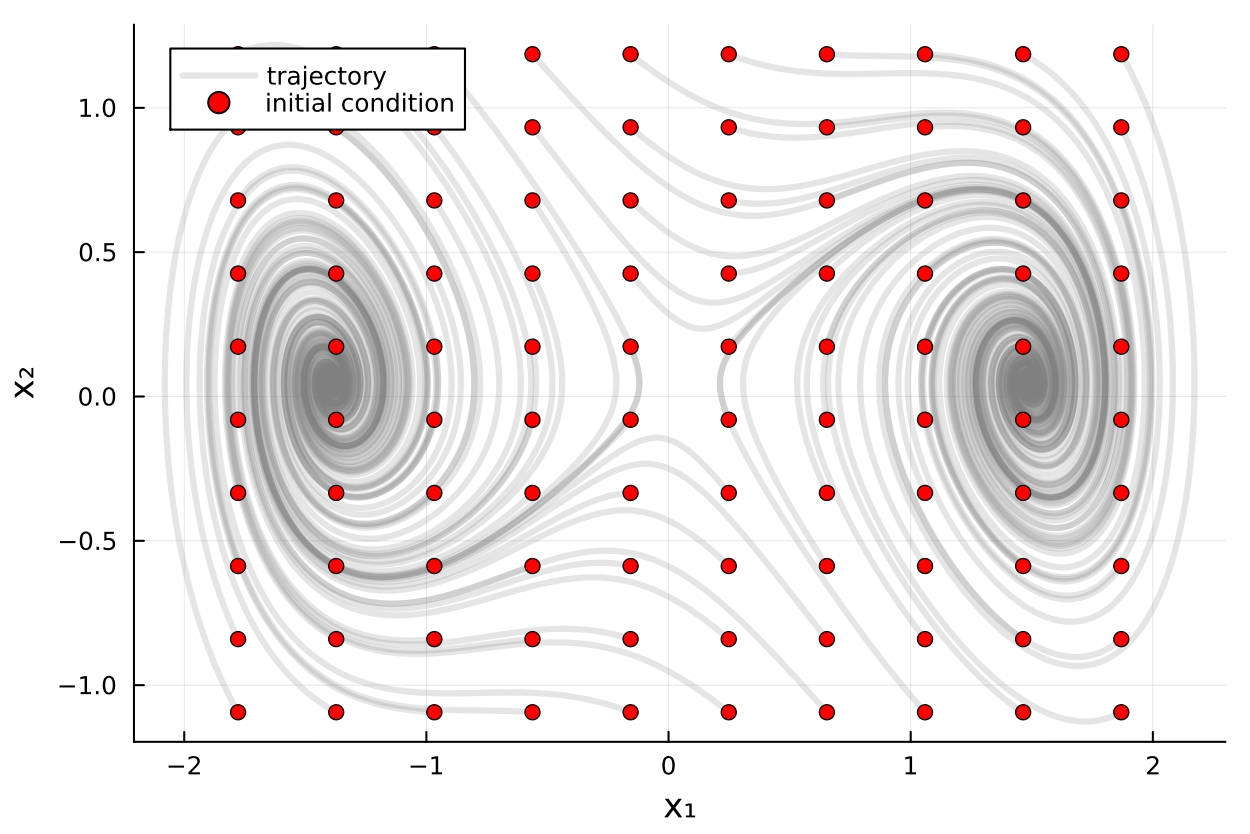}
    \caption{Grid $\textrm{Gr}(U)$ used to compute $\mathcal{E}_{stat}^U$ as well as $\mathcal{E}_{top}^U$ for the Duffing system.}
    \label{fig:duffing_U_grid}
\end{figure}

\section{Further Details on Section 4}
\subsection{Ground-truth Models}

\paragraph{Duffing system}  \label{appx:duffing}
The unforced Duffing system \citep{duffing1918erzwungene} is given by a set of coupled ODEs:
\begin{align}\label{eq:duffing_vf}
\dot{x} &= y \\ \nonumber
\dot{y} &= a y - x\left( b+ c x^2\right) 
\end{align}
where $[a,b,c]=[-\frac{1}{2},-1,\frac{1}{10}]$ places the system into a multistable regime with two coexisting equilibrium points. To generate datasets, we numerically integrate Eq. \eqref{eq:duffing_vf} for $t_{int}=40.0$ time units with a read-out interval of $\Delta t = 0.01$ using the adaptive step size integrator \texttt{Tsit5} provided within the Julia library \texttt{DifferentialEquations.jl} \citep{rackauckas2017differentialequations}. Using $K$ initial conditions, this results in an array of shape $4000 \times 2 \times K$. To facilitate training, we standardize our datasets by the overall mean and standard deviation across all trajectories.

\paragraph{Multistable Lorenz-like system}
As an example system of multistable chaotic attractors, we use the multistable Lorenz-like system introduced in \citet{lu_new_2004}:
\begin{align} \label{eq:multistable_lorenz}
    \begin{array}{l}{{\dot{x}=-{\frac{a b}{a+b}}x-y z+c}}\\ {{\dot{y}=a y+x z}}\\ {{\dot{z}=b z+x y\,,}}\end{array}
\end{align}
where we chose parameters $[a,b,c]=[-10,-4,18.1]$ such that the system exhibits two chaotic 1-scroll attractors in state space. We numerically integrate Eq. \eqref{eq:multistable_lorenz} for $t_{int}=80.0$ time units with a read-out interval of $\Delta t = 0.005$ using the \texttt{Tsit5} integrator. Using $K$ initial conditions, this results in an array of shape $16000 \times 3 \times K$. As for the Duffing datasets, we standardize using the overall mean and standard deviation.

\subsection{DSR Models and Training Routines}

\paragraph{SINDy} \label{appx:sindy}
Sparse Identification of Nonlinear Dynamics (SINDy) 
\cite{brunton_discovering_2016} aims to identify a sparse representation of the governing dynamical equations from data. Given a set of measurements of the state $\mathbf{x}(t) \in \mathbb{R}^n$, where $n$ is the number of system variables and $t=\{t_1 \dots t_m \}$ represents observation times, application of SINDy first requires approximating the flow $\frac{d\mathbf{x}}{dt}=\dot{\mathbf{x}}$ numerically, e.g. by finite difference methods. Following the notation in \citet{brunton_discovering_2016}, the derivatives are arranged into matrix form:
\begin{align}
\dot{\mathbf{X}} =
\begin{bmatrix}
\mathbf{\dot{x}}^\top(t_1) \\
\mathbf{\dot{x}}^\top(t_2) \\
\vdots \\
\mathbf{\dot{x}}^\top(t_m)
\end{bmatrix}
\begin{bmatrix}
\dot{x}_1(t_1) & \dot{x}_2(t_1) & \cdots & \dot{x}_n(t_1) \\
\dot{x}_1(t_2) & \dot{x}_2(t_2) & \cdots & \dot{x}_n(t_2) \\
\vdots         & \vdots         & \ddots & \vdots         \\
\dot{x}_1(t_m) & \dot{x}_2(t_m) & \cdots & \dot{x}_n(t_m)
\end{bmatrix}.
\end{align}

SINDy optimization then tries to determine a sparse matrix of regression coefficients $\boldsymbol{\Xi}$ such that:
\begin{equation}
    \dot{\mathbf{X}} = \Theta(\mathbf{x}) \boldsymbol{\Xi},
\end{equation}

Here, $\Theta(\mathbf{x})$ is a library of candidate functions on the state variables $\mathbf{x}$ that is defined beforehand, e.g.:
\begin{align}
\Theta(\mathbf{X}) =
\begin{bmatrix}
| & | & | & | & | & | & \\
\mathbf{1} & \mathbf{X} & \mathbf{X}^2 & \mathbf{X}^3 & \mathbf{X}^4 & \cos(\mathbf{X}) & \ldots \\
| & | & | & | & | & | & \\
\end{bmatrix},
\end{align}
The regression coefficients are found by applying a sparsity-promoting optimization technique, such as the least absolute shrinkage and selection operator (LASSO regression) or the Sequentially Thresholded Least Squares (STLSQ) algorithm, to solve for $\boldsymbol{\Xi}$.

The VF used for Figure \ref{fig_sindy_cycles} is defined as:
\begin{align} \label{eq:vf_sindy}
    \dot{x} &= x + x(x^2 + y^2 - 1)(4x^2 - 4xy + 4y^2) + (x^2 + y^2)(-2x + 2y + x^3 + xy^2), \\ \nonumber
    \dot{y} &= y + y(x^2 + y^2 - 1)(4x^2 - 4xy + 4y^2) + (x^2 + y^2)(-2x - 2y + y^3 + x^2y).
\end{align}
For the inner cycle, a trajectory was drawn from a randomly chosen initial condition $(x_0,y_0) = (0.6, 0.4)$. A long trajectory was then sampled with $T=100$ and $\Delta t=0.01$.
To infer the VF with SINDy, we used the Python implementation (\texttt{PySINDy}, \citet{de_silva_pysindy_2020}) with STLSQ optimizer and threshold $0.01$, and a \texttt{PolynomialLibrary} up to degree $6$. As the outer cycle is an unstable solution, small perturbations lead the system to diverge away from it, and so for the reconstructions in Fig. \ref{fig_sindy_cycles} (center) we made sure to initialize exactly on that cycle.
For the results in Figure \ref{fig_sindy_multistable}, we provided a polynomial library of second order for the multistable Lorenz-like system and a library of third order for the Duffing system, with other settings the same as used for Fig. \ref{fig_sindy_cycles}.

\paragraph{RNNs} \label{sec:supp:rnns}
We trained clipped shallow piecewise-linear RNNs (shPLRNNs; \citet{hess_generalized_2023}) using Backpropagation through time (BPTT) with sparse teacher forcing (STF, \citet{mikhaeil_difficulty_2022, brenner_tractable_2022}) and identity teacher forcing (id-STF; \citet{brenner_tractable_2022}). The clipped shPLRNN has a simple 1-hidden-layer architecture
\begin{equation}\label{eq:bounded_shPLRNN}
    \bm{z}_t  
    = \bm{A} \bm{z}_{t-1} + \bm{W}_1 \big[\phi(\bm{W}_2\bm{z}_{t-1} + \bm{h}_2) - \phi\left(\bm{W}_2 \bm{z}_{t-1}\right)\big] + \bm{h}_1,
\end{equation}
with latent state $\bm{z}_t \in \mathbb{R}^M$, diagonal matrix $\bm{A} \in \mathbb{R}^{M\times M}$, rectangular connectivity matrices $\bm{W}_1 \in \mathbb{R} ^{M\times H}$ and $\bm{W}_2 \in \mathbb{R} ^{H\times M}$, and thresholds $\bm{h}_2 \in \mathbb{R}^{H}$ and $\bm{h}_1 \in \mathbb{R}^{M}$. The nonlinear activation function $\phi$ is given by the $\textrm{ReLU}(\bullet) = \max(\bullet, 0)$. 

The idea behind id-STF is to replace latent states with states inferred from the observations at optimally chosen intervals $\tau$, such as to pull model-generated trajectories back on track and to avoid strong gradient divergence for chaotic dynamics \cite{mikhaeil_difficulty_2022}. Teacher forcing also has the effect of smoothening the loss landscape \cite{hess_generalized_2023}. As in \citet{brenner_tractable_2022}, we take an \say{identity-mapping} for the observation model, $\hat\vx_t = \mathcal{I}\mathbf{z}_t $, with $\mathcal{I} \in \mathbb{R}^{N \times M}$ and $\mathcal{I}_{kk}=1$ taken to be the identity for the $k$ read-out neurons, $k\leq N$, and zeros for all other elements. STF is only used in training the model, but not when deploying and testing it. The loss function to be minimized is the MSE:
\begin{equation}\label{eq:mse_loss}
    \ell_{MSE}(\hat{\bm{X}}, \bm{X}) = \frac{1}{NT_{s}}\sum_{t=1}^{T_{s}} \left\lVert \hat{\bm{x}}_t - \bm{x}_t \right\rVert_2^2,
\end{equation}
where $\hat{\bm{X}}$ are model predictions and $\bm{X}$ is the a training sequence of length $T_s$. For performing SGD updates, we employ the RAdam \citep{liu2019radam} optimizer paired with an exponential decay learning rate schedule. The shPLRNN and training routine are implemented using the \texttt{Flux.jl} DL stack \citep{Flux.jl-2018}. Detailed hyperparameter settings are collected in Table \ref{tab:hypers_shplrnn}.

\begin{table}[h]
\centering
\begin{tabular}{|c|c c|}
\hline
\textbf{Hyperparameter} & \textbf{Duffing} & \textbf{Lorenz-like} \\ \hline
$M$             & $5$          & $30$   \\
$H$             & $100$          & $500$   \\ 
$\tau$       & $15$         & $15$   \\ 
$T_{s}$        & $100$          & $50$   \\
batch size      & $32$          & $32$   \\
$\eta_{\textrm{start}}$                & $10^{-3}$          & $10^{-3}$   \\      
$\eta_{\textrm{end}}$            & $10^{-6}$          & $10^{-5}$   \\ 
\# \ trainable parameters            & $1,116$         & $30,641$   \\ 
\# \ SGD steps        & $250,000$          & $250,000$   \\ \hline
\end{tabular}
\caption{Hyperparameter settings of shPLRNNs trained on the Duffing and Lorenz-like systems.}
\label{tab:hypers_shplrnn}
\end{table}

\paragraph{Reservoir Computing (RC)}
We used a formulation of the RC architecture often employed in work on DSR \citep{patel_using_2022}: 
\begin{align}\label{eq:rc_model}
    \bm{r}_t &= \alpha \bm{r}_{t-1} + (1-\alpha) \tanh\left(\bm{W}\bm{r}_{t-1} + \bm{W}_{in}\bm{u}_{t} + \bm{b}\right) \\
    \hat{\bm{x}}_t &= \bm{W}_{out}\bm{r}_t,
\end{align}
where $\bm{r}_t \in \mathbb{R}^M$ is the reservoir state, $\alpha \in \mathbb{R}$ the leakage parameter, $\bm{W} \in \mathbb{R}^{M \times M}$ the reservoir connectivity matrix, $\bm{W}_{in} \in \mathbb{R}^{M \times N}$ the input-to-reservoir matrix weighing inputs $\bm{u}_t \in \mathbb{R}^N$, $\bm{b} \in \mathbb{R}^M$ a bias vector, and $\bm{W}_{out} \in \mathbb{R}^{N \times M}$ the matrix mapping reservoir states to the observed data. In RCs, the dynamical reservoir parameters $\theta_r = \{\alpha, \bm{W}, \bm{W}_{in}, \bm{b}\}$ are fixed after initialization. Here we initialized $\bm{W}$ to be fully connected with entries sampled from a standard normal distribution, and then scaled to have a predefined spectral radius specified by a hyperparameter $\rho$. Input-to-reservoir matrix $\bm{W}_{in}$ and bias $\bm{b}$ are also drawn from Gaussian distributions with variances $\sigma^2$ and $\beta^2$, respectively. In RCs, only the reservoir-to-output matrix $\bm{W}_{out}$ is learned. The RC is trained by first driving the reservoir using ground-truth data $\bm{X} = \left[\bm{x}_1, \dots, \ \bm{x}_T\right] \in \mathbb{R}^{N \times T}$ supplied through $\bm{u}_t = \bm{x}_t$. This results in a trajectory of reservoir states $\bm{R} = \left[\bm{r}_1, \dots, \ \bm{r}_T\right] \in \mathbb{R}^{M \times T}$. The only trainable parameters $\bm{W}_{out}$ are then determined by minimizing the least-squares error $\lVert\bm{X} - \bm{W}_{out}\bm{R}\rVert_2^2$, a straightforward linear regression problem with closed form solution
\begin{equation}
    \bm{W}_{out} = \bm{X} \bm{R}^T \left(\bm{R}\bm{R}^T\right)^{-1}.
\end{equation}
After training, the reservoir state is initialized with zeros and the RC is only provided a short sequence of ground-truth data $\{\vx_1, \dots, \vx_{T_W}\}$ to 'warm-up` the dynamics of the reservoir state $\bm{r}_t$, where $T_W$ denotes the warm-up time. Afterwards, the RC runs closed-loop (autonomously) by feeding predictions $\hat{\bm{x}}_t$ back to the reservoir through the input-to-reservoir connection. To keep the comparison between DSR models fair in Fig. \ref{fig_bench}, we only provide a single initial condition, i.e. $T_W = 1$. For visual clarity, however, we still drop the first few time steps of RC-generated trajectories (e.g. in Fig. \ref{fig_bench}), which the zero-initialized reservoir state needs to converge to the correct dynamics. Detailed hyperparameter settings are in Table \ref{tab:hypers_rc}.

\begin{table}[h]
\centering
\begin{tabular}{|c|c c|}
\hline
\textbf{Hyperparameter} & \textbf{Duffing} & \textbf{Lorenz-like} \\ \hline
$M$           & $500$          & $2000$   \\
$\rho$      & $1.0$          & $0.75$   \\ 
$\alpha$      & $0.7$          & $0.4$   \\ 
$\sigma$       & $0.2$         & $0.3$   \\ 
$\beta$    & $0.5$          & $0.7$   \\
\# \ trainable parameters        & $1,000$          & $6,000$   \\ \hline
\end{tabular}
\caption{Hyperparameter settings of RCs trained on the Duffing and Lorenz-like systems.}
\label{tab:hypers_rc}
\end{table}

\paragraph{N-ODE}
We train N-ODEs \citep{chen2018neural} using the Julia library \texttt{DiffEqFlux.jl} \citep{rackauckas2020universal}. We use a simple multi-layer perceptron (MLP) architecture where parameters are optimized using the adjoint method. The loss function is the MSE, Eq. \eqref{eq:mse_loss}. As for RNNs we perform SGD updates using RAdam paired with an exponential decay learning rate schedule.  Detailed hyperparameter settings are in Table \ref{tab:hypers_node}.
\begin{table}[h]
\centering
\begin{tabular}{|c|c c|}
\hline
\textbf{Hyperparameter} & \textbf{Duffing} & \textbf{Lorenz-like} \\ \hline
\# \ hidden layer            & $2$          & $3$   \\
hidden layer sizes             & $[40, 40]$          & $[100, 100, 100]$   \\ 
activation       & tanh         & ReLU   \\ 
$T_{s}$        & $30$          & $30$   \\
batch size      & $32$          & $32$   \\
ODE solver & Tsit5 & Tsit5   \\
$\eta_{\textrm{start}}$                & $10^{-3}$          & $10^{-3}$   \\   
$\eta_{\textrm{end}}$            & $10^{-5}$          & $10^{-5}$   \\ 
\# \ trainable parameters            & $1,842$         & $20,903$   \\ 
\# \ SGD steps        & $100,000$          & $100,000$   \\ \hline
\end{tabular}
\caption{Hyperparameter settings of N-ODEs trained on the Duffing and Lorenz-like systems.}
\label{tab:hypers_node}
\end{table}

\begin{figure*}[htb!]
  \centering  \includegraphics[width=1.0\textwidth]{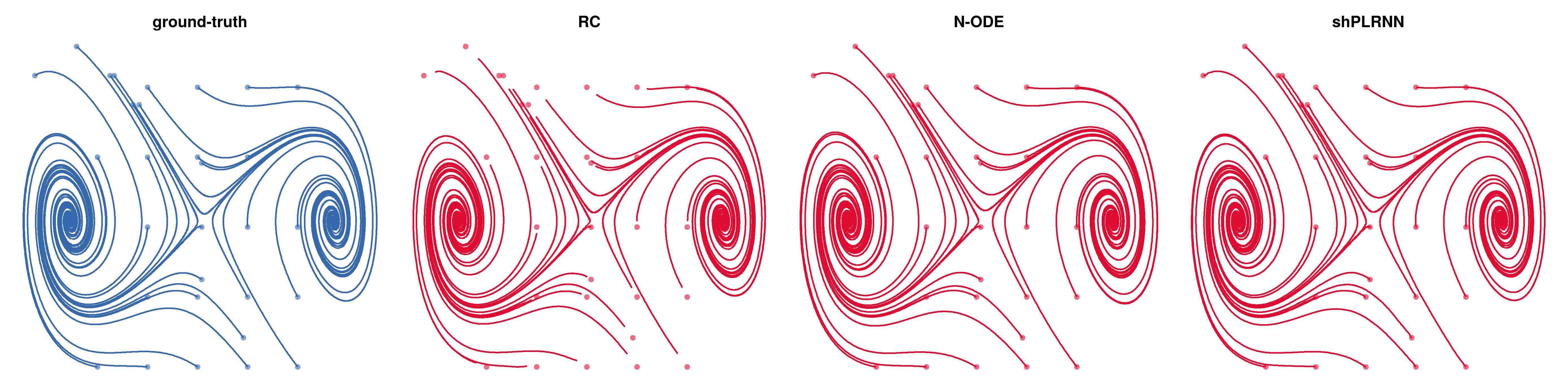}
  \caption{Reconstructions of the Duffing system as in Fig. \ref{fig_bench}, but with models trained on ground-truth data (blue trajectories) from both basins. The models are capable of learning the multistable dynamics (red trajectories) of the system when supplied with data from both basins. This is also reflected in the drastically smaller statistical error $\mathcal{E}_{stat}$ when applied to the same test data as used for Fig. \ref{fig_bench}: $\textrm{RC} \approx 2.7\cdot 10^{-3}$, $\textrm{N-ODE} \approx 2.1 \cdot 10^{-3}$ and $\textrm{shPLRNN} \approx 1.4\cdot 10^{-3}$.}
  \label{fig_bench_duffing_full}
  \normalsize
\end{figure*}

\begin{figure*}[htb!]
  \centering  \includegraphics[width=1.0\textwidth]{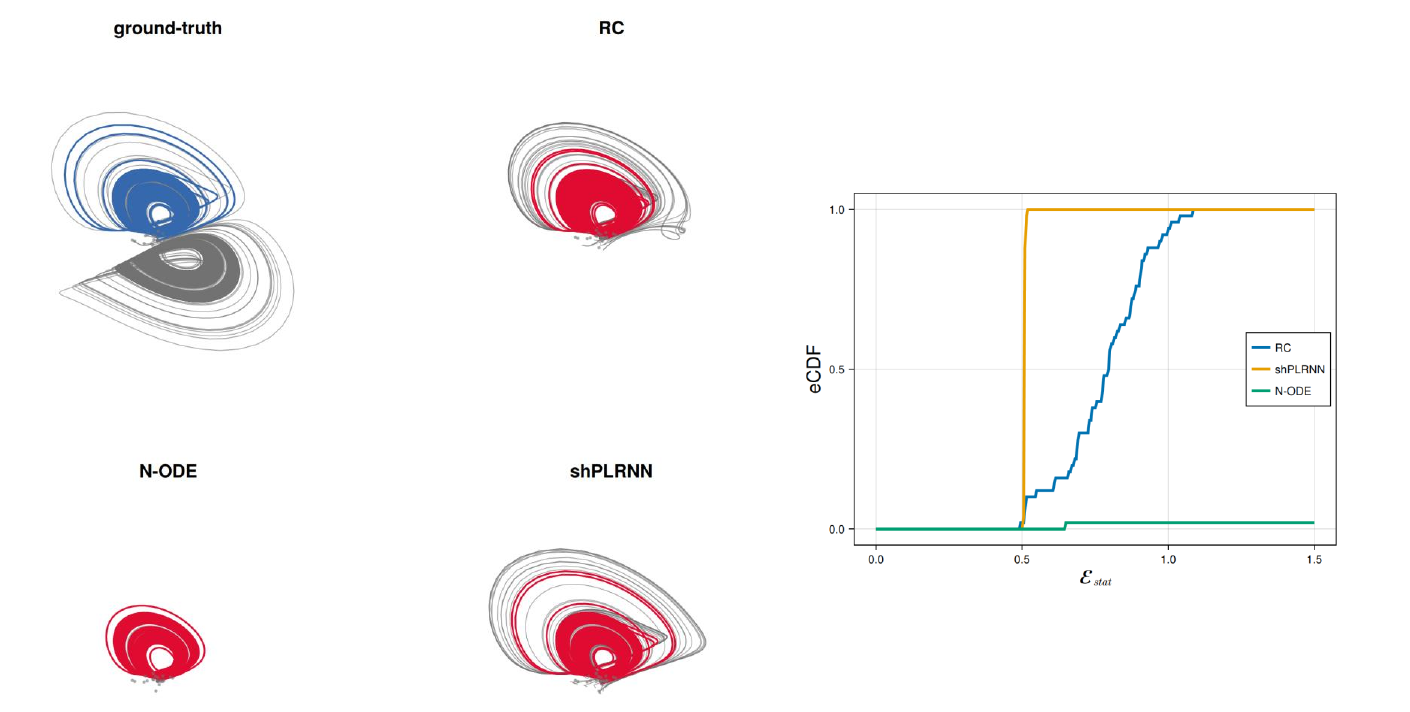 }
  \caption{Learnability evaluated on the multistable Lorenz-like system. Left: Example reconstructions of DSR models trained on 4 ground-truth trajectories (blue) from one basin. Red trajectories are freely generated using initial conditions of the training data and the respective DSR model. Grey trajectories are example ground-truth test trajectories and model-generated trajectories, respectively, from both the training basin and the OOD basin. Again, all models fail to properly generalize to the unobserved attractor/basin. Right: eCDF of $\mathcal{E}_{stat}$ with a sample size of $N=50$ independent trainings of each DSR model evaluated over a grid of $125$ initial conditions covering both basins. Note that the dynamics of the N-ODE models consistently diverged for many initial conditions from the grid. For N-ODE $\mathcal{E}_{stat}$ values, this means that most mass is concentrated at much higher error values, which were cut off in the eCDF plot.}
  \label{fig_bench_lorenz_multi}
  \normalsize
\end{figure*}

\begin{figure*}[htb!]
  \centering  \includegraphics[width=0.6\textwidth]{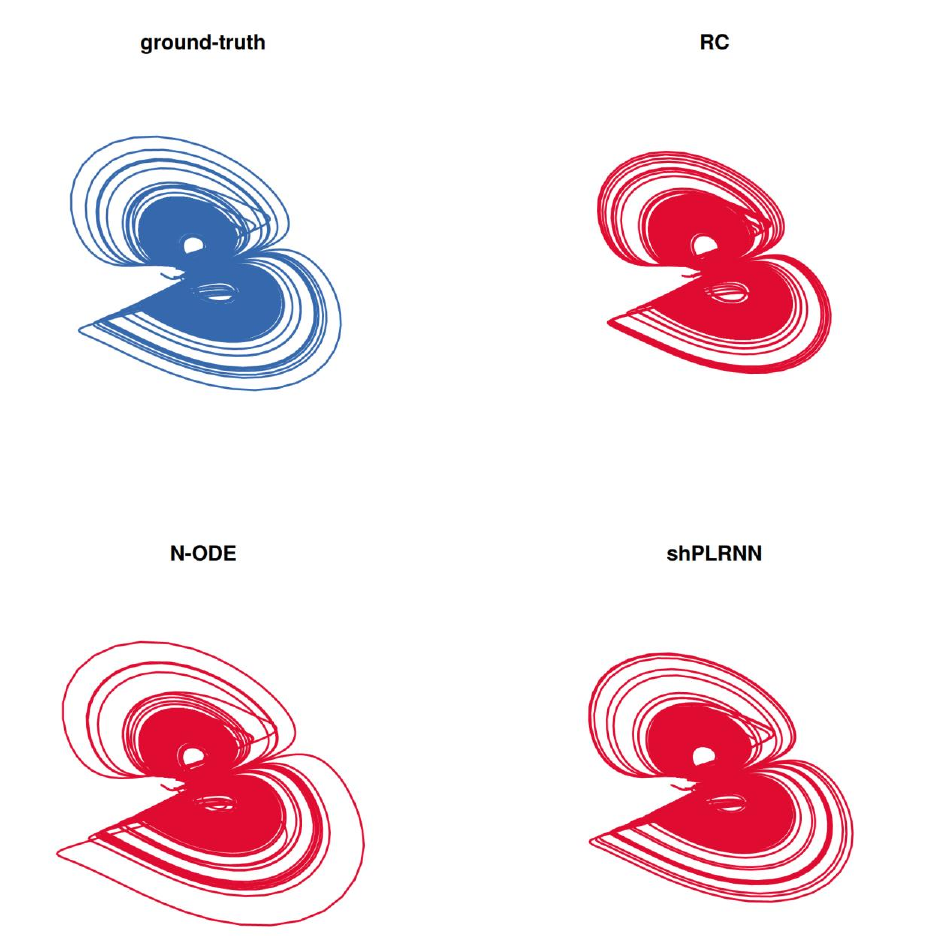}
  \caption{Reconstructions of the multistable Lorenz-like system using the same DSR models as in Fig. \ref{fig_bench_lorenz_multi}, but trained on ground-truth training data (blue) from both basins. As for the Duffing system, the models are capable of learning the multistable dynamics (red trajectories) of the Lorenz-like system when supplied with data from both basins. This is also reflected in lower statistical errors $\mathcal{E}_{stat}$ when compared to the monostable training data, evaluating to $\approx 0.131$ for RC, $\approx 0.133$ for N-ODE and $\approx 0.064$ for the shPLRNN.}
  \label{fig_bench_lorenz_full}
  \normalsize
\end{figure*}

\begin{figure*}[htb!]
  \centering  \includegraphics[width=.3\textwidth]{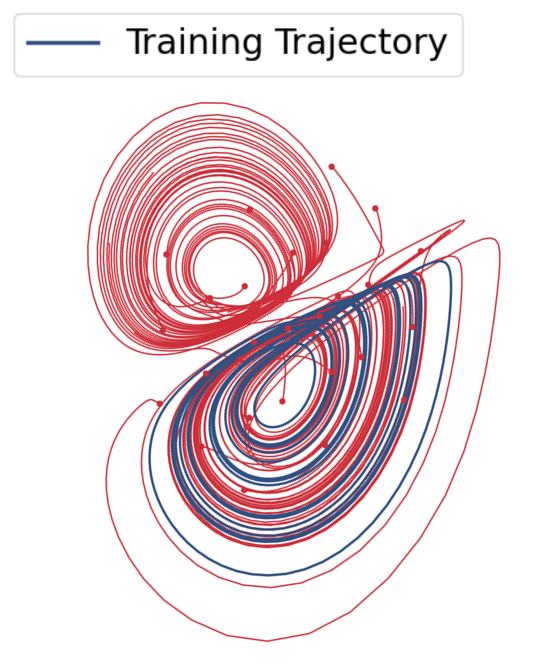}
  \caption{Reconstruction of the multistable Lorenz-like system from a trajectory from just one basin, using PySINDy \cite{de_silva_pysindy_2020}. Since in this case, the correct polynomial function library was provided, both basins are correctly identified (see also Fig. \ref{fig_sindy_cycles}b).}
  \label{fig_sindy_multistable}
  \normalsize
\end{figure*}

\begin{figure}[h!]
  \centering 
  \includegraphics[width=0.99\textwidth]{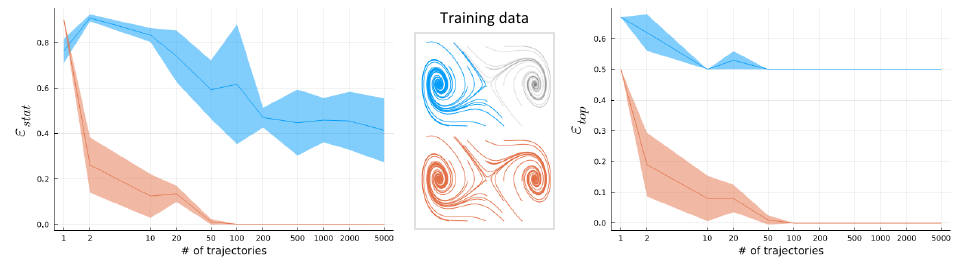}
  \caption{Generalization errors $\mathcal{E}_{\mathrm{stat}}^M$ (left) and $\mathcal{E}_{\mathrm{top}}^M$ (right) for shPLRNNs ($M=5$, $H=100$) as a function of sample size (amount of training data), drawn from one basin (blue) or both basins (orange) of the Duffing system (center). Each data point is median $\pm$ median absolute deviation (ribbon) across 10 models. Increasing training data from one basin makes no difference for the generalization to the second basin (blue), i.e. both errors plateau at high values.
  }
  \label{fig_error_vs_sample_size}
\end{figure}

\newpage
~\newpage
\subsection{Specifications on Theorem \ref{th_sc1}}
\label{appx_specth1}
There are also examples of VFs with a dense set of trajectories which solve an algebraic equation. Any VF with a rational first integral (e.g. algebraic Hamiltonian) is not learnable (since then every solution solves an algebraic equation in terms of the basis functions). A simple example would be the standard harmonic oscillator in the regime where it has a center, i.e. a dense set of closed orbits (each of them solving an algebraic equation). More generally, systems with a center manifold (as in many
Hamiltonian systems and biological systems like the FitzHugh-Nagumo equation) may have this property, with all trajectories lying on the center manifold solving an algebraic equation.

\begin{figure}[h!]
 \centering \includegraphics[width=0.7\textwidth]{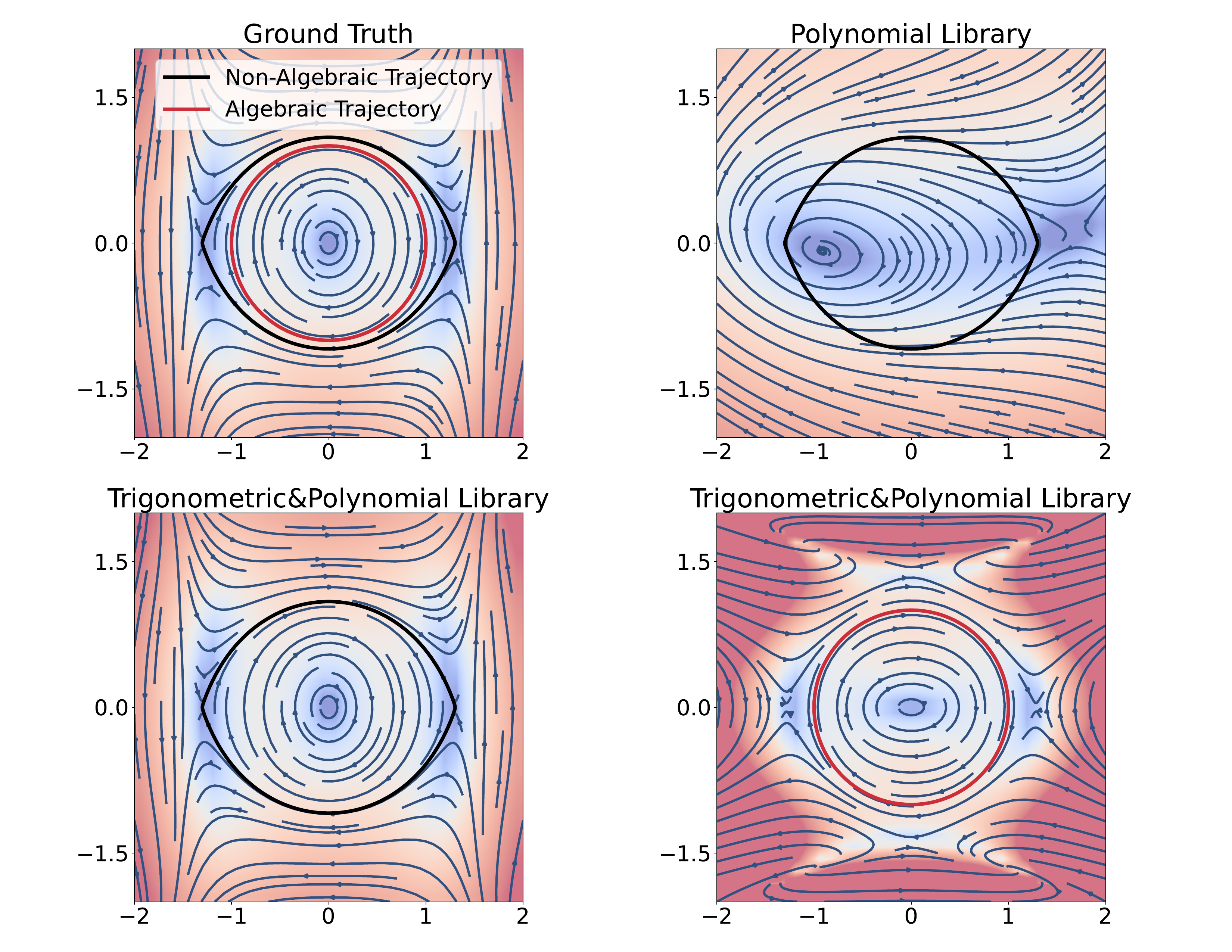}
  \caption{Example reconstructions using SINDy (details in Appx. \ref{appx:sindy}). a) The ground truth (GT) VF has two cycle solutions. One solves an algebraic equation (red), while the other does not (black). b) Providing SINDy with the black trajectory and the correct library (including both trigonometric and polynomial functions) leads to a correctly inferred VF. c) Providing as data the curve solving an algebraic trajectory leads to an incorrectly inferred VF, as stated in theorem \ref{th_sc1}, despite the correct library. d) Providing SINDy with a non-algebraic trajectory but a library that lacks the trigonometric terms also fails to reproduce the GT VF.} \label{fig_sindy_trigo}
\end{figure}
The VF used for Figure \ref{fig_sindy_trigo} is defined as:

\begin{align} \label{eq:vf_sindy}
  \dot{x} &= 2y \cos(x), \\ \nonumber
\dot{y} &= x^2 \sin(x) - 2x \cos(x) + y^2 \sin(x) - \sin(x).
\end{align}

\newpage
\subsection{Specifications on Theorem \ref{th_sc1} - Evaluating Identifiability Conditions}

\label{appx_para_id}
 
Assume we are given a dataset (trajectory) 
\begin{align} \mathcal{D}=\{\vx(t_0),\vx(t_1),...,\vx(t_N)\} ,\end{align} for which we would like to check whether it satisfies an algebraic equation in the basis functions, i.e.,
\begin{align}\label{eq_al}
     \nexists \vtheta \in \mathbb{R}^{m } \backslash \{\bm{0}\}: \ \Lambda(\vx(t))=\sum_{i=0}^{m}\theta_{i}\psi_{i}(\vx(t))=0 \quad \forall t \in I_{\vx_0}^f .
\end{align}
Multiplying Eq. \eqref{eq_al} by $\psi_j(\vx(t))$ yields 
\begin{align}-
   \forall j=0,...,m: \quad  \psi_j(\vx(t)) \sum_{i=0}^{m}\theta_{i}\psi_{i}(\vx(t))=0.
\end{align}
As this equation is always zero, it is also zero when evaluated at the data points
\begin{align}\label{eq:linear_equations_index_not}
    \forall j=0,...,m: \quad  \sum_{k=0}^N\psi_j(\vx(t_k)) \sum_{i=0}^{m}\theta_{i}\psi_{i}(\vx(t_k))=0.
\end{align}
Equation \eqref{eq:linear_equations_index_not} consists of $m+1$ linear equations which can be written in matrix form:
\begin{equation} \label{eq_appx_big}
	\underbrace{\sum_{k=0}^N\begin{bmatrix}  \psi_{0}(\vx(t_k)) \cdot \psi_{0}(\vx(t_k))   &   \psi_{0}(\vx(t_k)) \cdot \psi_{1}(\vx(t_k))   & \cdots  &   \psi_{0}(\vx(t_k)) \cdot \psi_{m}(\vx(t_k))   \\   \psi_{1}(\vx(t_k)) \cdot \psi_{0}(\vx(t_k))   &   \psi_{1}(\vx(t_k)) \cdot \psi_{1}(\vx(t_k))   &  &   \psi_{1}(\vx(t_k)) \cdot \psi_{m}(\vx(t_k))   \\ \vdots  & \vdots  & \ddots  & \vdots  \\   \psi_{m}(\vx(t_k)) \cdot \psi_{0}(\vx(t_k))   &   \psi_{m}(\vx(t_k)) \cdot \psi_{1}(\vx(t_k))   &   \dots   &   \psi_{m}(\vx(t_k)) \cdot \psi_{m}(\vx(t_k))   \end{bmatrix}}_{\underline{\Psi}}\begin{bmatrix}  \theta_0 \\\theta_1 \\\vdots \\ \theta_m \end{bmatrix} =0.
\end{equation}
Hence, if $\operatorname{ker}(\underline\Psi)$ is non-trivial (contains not only the zero-vector), then the data points solve an algebraic equation in the basis functions. For instance, let us revisit the scenario described in the main text, where our training dataset comprises points on the circle $\Gamma_{\vx_0}=\big\{ (\cos(t),\sin(t)) \mid t \in [0, 2 \pi) \big\} $. With this information, we are able to determine the null space of $\underline{\Psi}$ by solving Eq. \eqref{eq_appx_big} using a library that includes polynomials up to third order. The null space is three-dimensional and consists of the three algebraic curves shown in Fig. \ref{fig:algebraic_check_example}. (Note that, by definition, also all linear combinations of these invariant algebraic curves are in the nullspace.) In contrast, the nullspace of $\underline{\Psi}$ for a trajectory that does not solve an algebraic equation in the basis functions contains only the zero-vector.\\

Thus, checking whether the null space of $\underline{\Psi}$ contains only the zero-vector is a quick and efficient method for verifying whether any trajectory satisfies an algebraic equation in the basis functions. If the nullspace contains only the zero-vector and we have provided a proper (correct) library, SINDy (or related library methods) will generalize across the state space $M$. Conversely, if the null space is non-trivial, either a different trajectory must be selected or the hypothesis class must be limited, as outlined in Corollary \ref{th_sc1_2}, to enable proper generalization.

\begin{figure}[h!]
  \centering \includegraphics[width=0.95\textwidth]{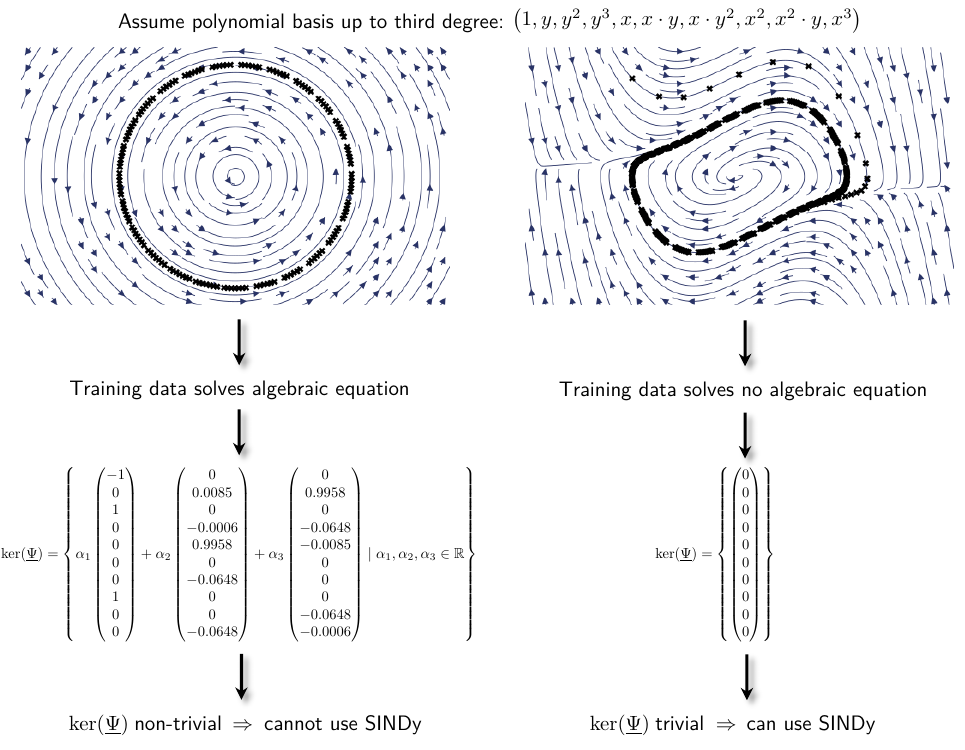}
  \caption{ Illustration of how to check whether training data solve, or do not solve, an algebraic equation in a practical setting, for the harmonic oscillator (left) and van-der-Pol oscillator (right). The limit cycle of the van-der-Pol oscillator is non-algebraic \cite{odani_limit_1995}.}
  \label{fig:algebraic_check_example}
\end{figure}

\newpage
\subsection{Why OODG Fails}
\label{appx_psgd} 
The probability of a model trained with SGD having error $\varepsilon_{\mathrm{stat}}$ is defined by
\begin{align}
    p_{\mathrm{SGD}}(\varepsilon_{\mathrm{stat}} \mid \mathcal{D}) = \int_{\Theta} \mathbbm{1}{[ \mathcal{E}_{\mathrm{\mathrm{gen}}}^M(\Phi_{\vtheta_f}) = \varepsilon_{\mathrm{\mathrm{gen}}}  ]} p_{\mathrm{opt}}\left(\vtheta_{f} \mid \vtheta_{i}, \mathcal{D}\right) p_{\mathrm{ini}}\left(\vtheta_{i}\right) d \vtheta_{i} d \vtheta_{f}.
\end{align}
This coincides with the learnability distribution
\begin{align}
    p(\varepsilon_{\mathrm{gen}}|\mathcal{D}) = \frac{1}{\mathrm{vol}(\Theta_0)} \int_{\Theta_0} \mathbbm{1}{[ \mathcal{E}^{M_{\mathrm{test}}}_{\mathrm{\mathrm{gen}}}(\Phi_{\vtheta}) = \varepsilon_{\mathrm{gen}}  ]} d \vtheta ,
\end{align}
under two conditions, the failure of either one introduces an implicit bias. 
First, it might be that the optimizer does not converge to a model with zero reconstruction error on $M$ but a slightly larger error. The more significant implicit bias arises from the combination of $p_{\mathrm{ini}}$ and $p_{\mathrm{stat}}$. Assuming the optimizer always converges to a model with zero generalization error, we have
\begin{align}
    p_{\mathrm{SGD}}(\varepsilon_{\mathrm{gen}}|\mathcal{D})= \frac{1}{\mathrm{vol}(\Theta_0)} \int_{\Theta_0} \mathbbm{1}{[\mathcal{E}^{M_{\mathrm{test}}}_{\mathrm{\mathrm{gen}}}(\Phi_{\vtheta}) = \varepsilon_{\mathrm{gen}}  ]} \cdot p_{\mathrm{bias}}(\vtheta) d \vtheta,
\end{align}
i.e., $p_{\mathrm{SGD}}$ aligns with the learnability distribution only when an implicit bias term $p_{\mathrm{bias}}$ is accounted for. In general, it is likely that a specific combination of $p_{\mathrm{ini}}$ and $p_{\mathrm{stat}}$ preferentially converges to certain parameters $\vtheta_f$, as observed for SGD. This preference is termed the implicit bias of the learning algorithm, and skews the learnability distribution toward the domain implicitly preferred by the learning algorithm. 
\subsection{Simplicity Bias}
\label{appx_simpbias}
We adopt the following definition for the parameter-function map from \cite{vallepérez2019deep}:
\begin{definition}
\label{appx_def_M}
The parameter-function map $\mathcal{M}$ associates a given parameter set with the flow, expressed as:
\begin{align}
    \mathcal{M}: \Theta \rightarrow \mathcal{H}_{\theta}, \quad \vtheta \mapsto {\Phi_{\theta}}.
\end{align}
\end{definition}

Depending on the choice of model, multiple parameter vectors can map onto the same flow. For instance, for the shPLRNN (Sect. \ref{sec:supp:rnns}), scaling $\bm{W}_1$ by a positive scalar factor $c \in \mathbb{R}_+$ and adjusting the weights in $\bm{W}_2$ by $\frac{1}{c}$ results in the same dynamical model. The initial parameter distribution $p_{\mathrm{ini}}$ can lead to flows with different characteristics not directly apparent from the initial distribution. In the context of the simplicity bias, we are interested in the distribution over the complexity of the flows $K(\Phi_{p_{\mathrm{ini}}})$ induced by a choice of distribution over initial parameters. To assess this complexity, we select the Shannon entropy over the limit sets of the resulting flows \cite{eckmann_ergodic_1985}. We evaluate this by drawing long trajectories from a grid of initial conditions and compute the entropy over the histogram of final states, using the R\'enyi algorithm from \texttt{ComplexityMeasures.jl}. This entropy has a natural interpretation for flows with different topologies, where e.g. global equilibrium points have low entropy and chaotic attractors have high entropy (see Fig. \ref{fig:mean_entropy_init}b).

\begin{figure}[H]
  \centering \includegraphics[width=0.5\textwidth]{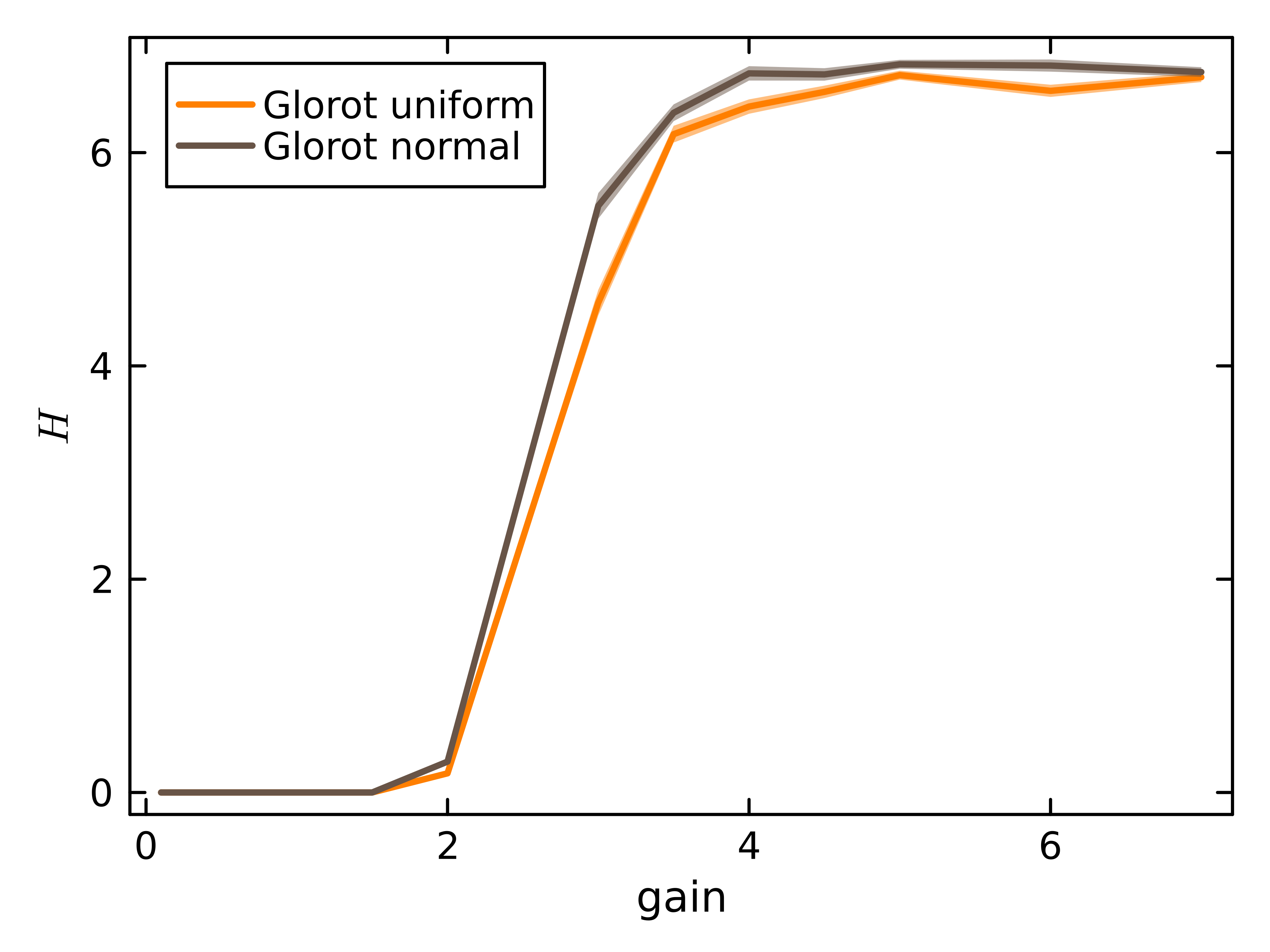}
  \caption{ Simplicity bias for shPLRNNs for $M=10$ and $H=250$. }
  \label{fig_exp}
\end{figure}

\begin{table}[ht!]
\centering
\caption{Number of positive, zero and negative eigendirections for the Hessian of the loss function evaluated on trajectories from just one ($\ell_{B(A_1)}$) or both ($\ell_{M}$) basins w.r.t. $\vtheta_{\mathrm{gen}}$. (For this analysis only 5 of the 20 generalizing models plotted in Fig. \ref{fig_relearning} where used.) \\}
\label{appx_table_eigendirections}
\begin{tabular}{@{}|l|l|l|@{}}
\toprule
                & $\ell_{M}$             & $\ell_{B(A_1)}$          \\ \midrule
$\#\lambda_+$   & $171.93 \pm 0.62$      & $103.08 \pm 2.10$       \\ \midrule
$\#\lambda_0$   & $182.38 \pm 0.75$      & $310.54 \pm 3.76$       \\ \midrule
$\#\lambda_-$   & $149.69 \pm 0.24$      & $90.38 \pm 1.66 $       \\ \midrule
$\lambda_{max}$ & $241.09 \pm 1.25$      & $364.09 \pm 3.29 $      \\ \midrule
$\lambda_{min}$ & $-0.00126 \pm 0.00004$ & $-0.00186 \pm 0.00001$  \\ \bottomrule
\end{tabular}
\end{table}

\begin{figure}[hb!]
\begin{center}
\includegraphics[width=0.99\textwidth]{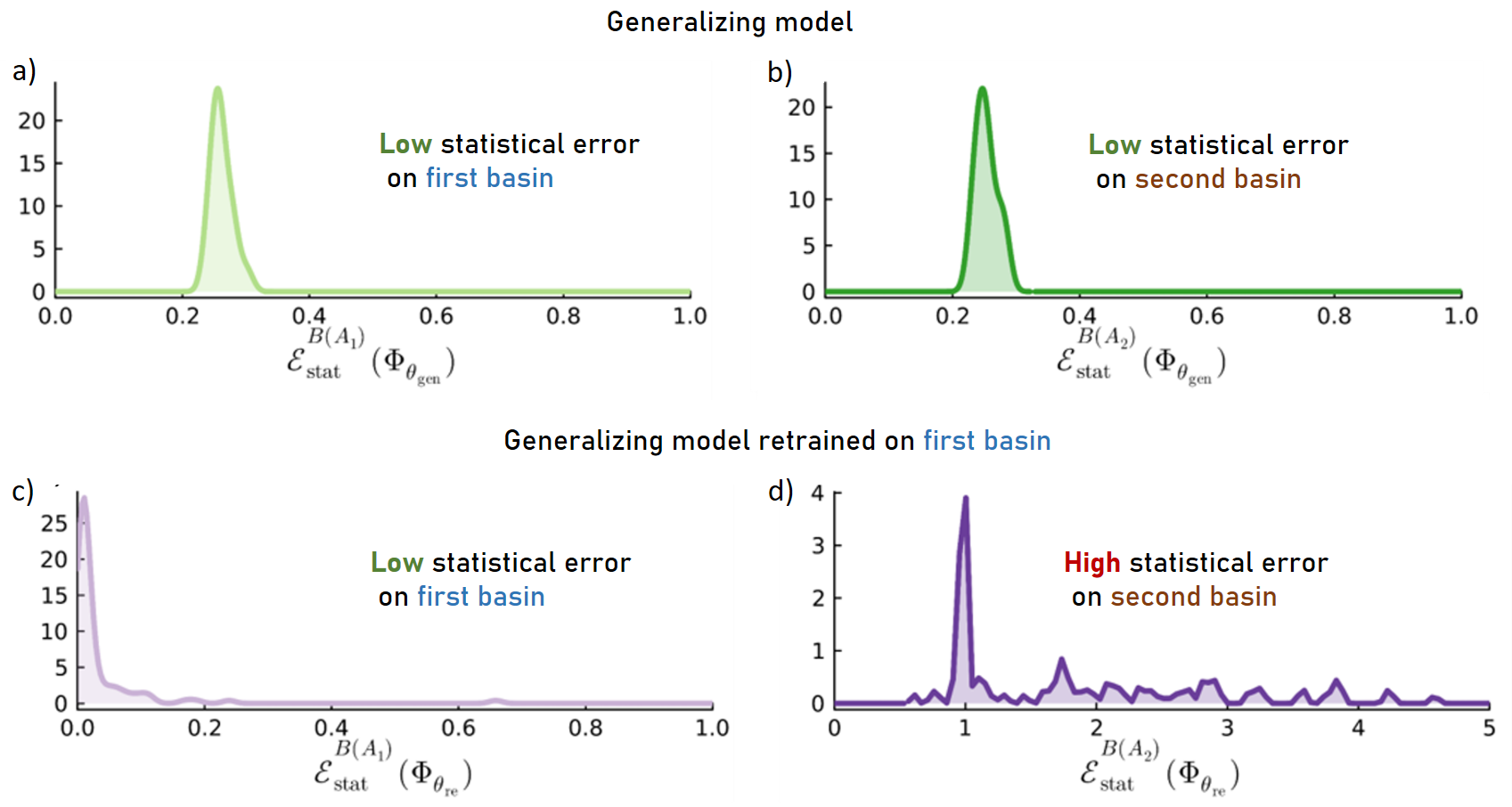}
\end{center}
\caption{Similar to Fig. \ref{fig_relearning}a, this figure illustrates the statistical error of generalizing and retrained models for the multistable Lorenz-like system described by Eq. \eqref{eq:multistable_lorenz}. The upper two density plots depict the statistical error of generalizing models on $B(A_1)$ and $B(A_2)$. Meanwhile, the lower row illustrates the same for retrained models. A surge in error is observed for $B(A_2)$ in the retrained models, suggesting an unlearning of the second attractor. This finding confirms that the results presented in the main text can be reproduced for a ground-truth system with completely different dynamics.}
\label{fig:multistable_lorenz_saddles}
\end{figure}

\begin{figure}[h!]
  \centering \includegraphics[width=1.0\textwidth]{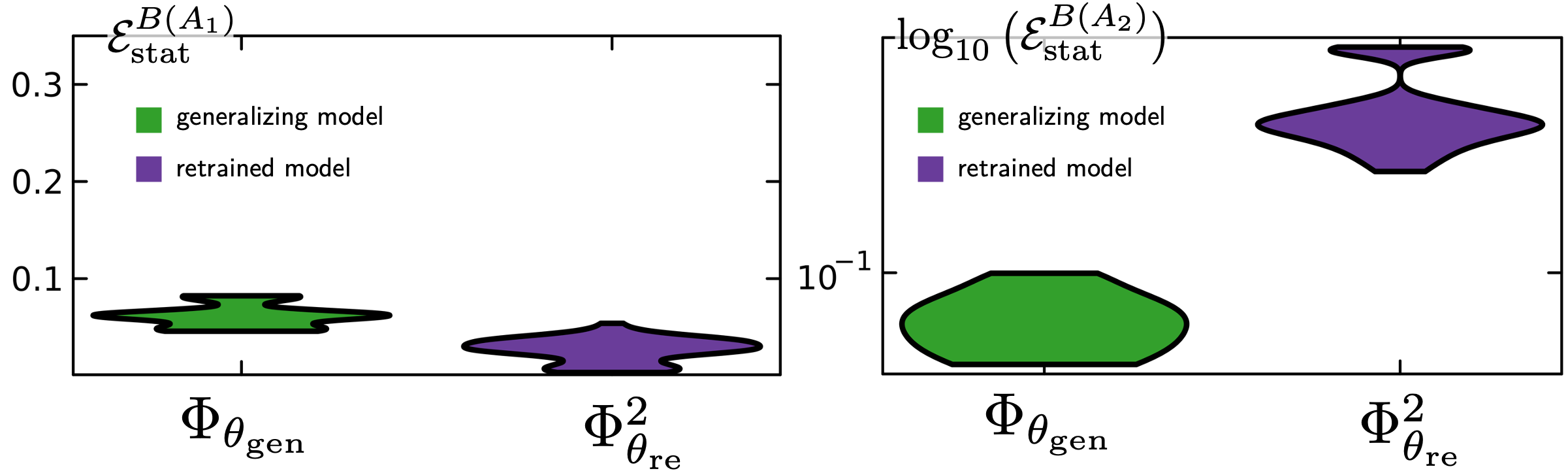}
  \caption{Statistical error distribution on basins $B(A_1)$ and $B(A_2)$ for $10$ generalizing models (green) and $10 \times 5$ models retrained (purple) using only $B(A_1)$ data, similar to Fig. \ref{fig_relearning}a in the main text, but based on a $6d$ Lorenz-96 system \cite{pelzer_finite_2020} of the form $x_j = x_{j-1} \cdot (x_j-x_{j-1})-x_j + F, \ j=1 \dots 6$, where periodic boundary conditions are applied. Using $F=0.654502$, the system has two coexisting chaotic attractors (see sect. 4.3 in \cite{pelzer_finite_2020} for more details on the system).  
  After retraining on data from $B(A_1)$, the models effectively unlearn the dynamics on $B(A_2)$.}
  \label{fig:relearn_l96}
\end{figure}

\subsection{Assessing Sharpness of Minima}
\label{appx_sharpmin}
In order to assess the volume of minima, we employ a sampling-based approach. Following \citet{huang_understanding_2020}, we randomly select a vector $\vtheta'$ from the parameter space $\Theta$ with dimensionality $d$ within the hyper-sphere $\mathcal{S}^{d-1}_r$ with radius $r$. Subsequently, we evaluate the loss of models with parameters along a straight line in parameter space given by
\begin{equation}
    \vtheta_{\mathrm{min}}+ a(\vtheta'- \vtheta_{\mathrm{min}}).
\end{equation}
Here, $a$ takes values in the interval $[0,1]$. We define the threshold value $a=a_{th}$ as the one at which the loss of the corresponding model exceeds the predefined threshold $\ell > \ell (1+p_{th})$. We tested a threshold of $p_{th}=1 \%$ and $p_{th}=5 \%$ for our experiments. However, as the results in Fig. \ref{fig_radius} and \ref{appx::minima_radii} indicate, the results obtained are not overly sensitive to this hyperparameter. The resulting threshold value $a_{th}$ yields an estimate of the minimum radius in the direction of $\vtheta'$, given by $r\cdot a_{th}$. This radius serves as a lower bound for the minimum volume $V=\frac{\pi^{n / 2}}{\Gamma(1+n / 2)} \mathbb{E}_{\vtheta}\left[r^{d}(\vtheta)\right]$, given by:
\begin{align*}
    \log V &= \frac{d}{2} \log \pi - \log \Gamma(\frac{d}{2} + 1) + \log \mathbb{E}_{\bm{\phi} \sim \mathcal{U}}[r^d(\vtheta)] \\
    &\geq \frac{d}{2} \log \pi - \log \Gamma(\frac{d}{2} + 1) + \mathbb{E}_{\bm{\vtheta} \sim \mathcal{U}}[d\log r(\bm{\vtheta})]\\
    &\approx \frac{d}{2} \log \pi - \log \Gamma(\frac{d}{2} + 1) + \frac{d}{N}\sum_{i=1}^N \log r(\bm{\vtheta}_i),
\end{align*}
where we used Jensen's inequality to pull the logarithm into the expectation, and $\Gamma(\frac{d}{2} + 1)$ is Euler's gamma function.

\begin{figure}[!htb]
  \centering \includegraphics[width=0.99\textwidth]{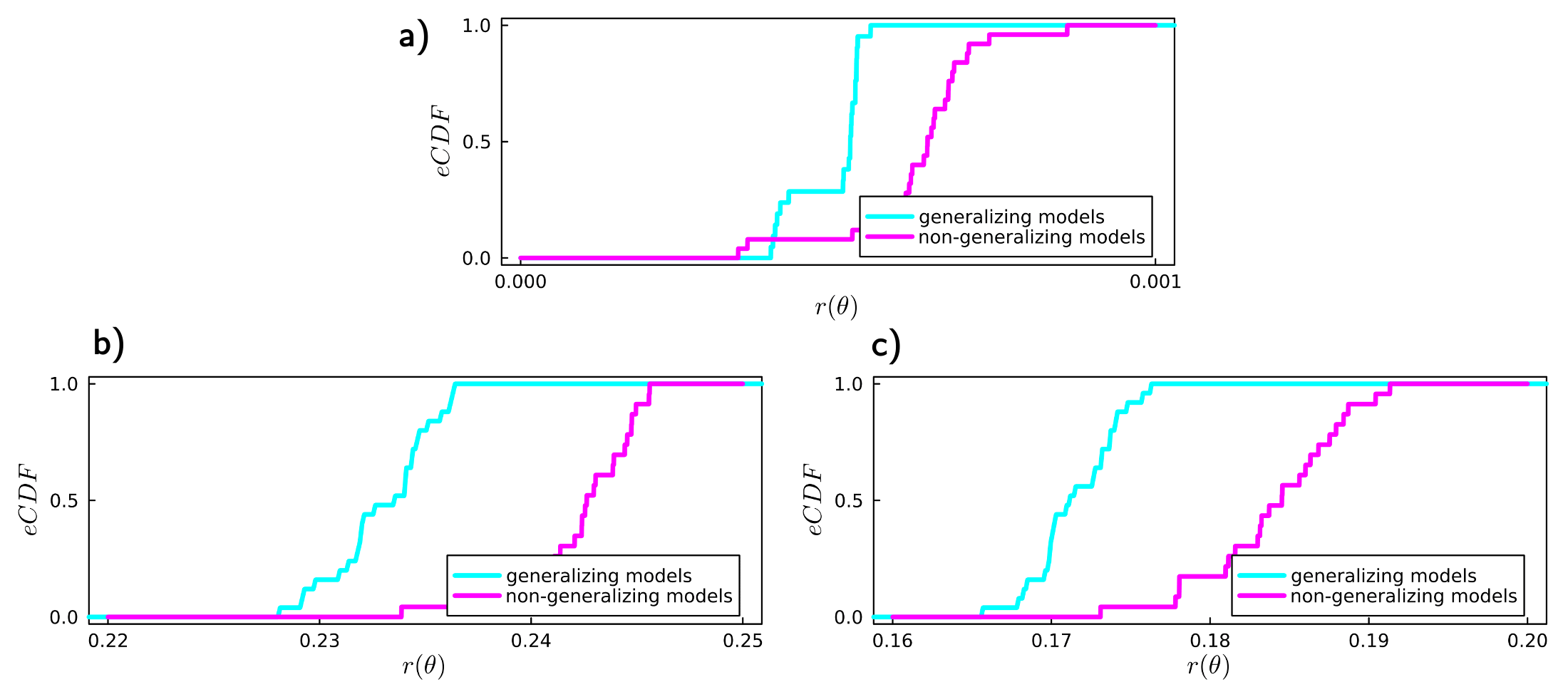}
  \caption{ \textbf{a)} Similar results as in Fig. \ref{fig_radius} for the Duffing system, but for a lower threshold $p_{th}=1 \%$, as described in Appx. \ref{appx_sharpmin}. \textbf{b)} Similar graphs as in Fig. \ref{fig_radius} for the multistable Lorenz system (Eq. \eqref{eq:multistable_lorenz}) for $p_{th}=1 \%$. \textbf{c)} Same as b) but for $p_{th}=5 \%$.}
  \label{appx::minima_radii}
  \normalsize
\end{figure}

\paragraph{Radius estimates for saddle points} \label{appx:saddles_radius}
It is widely acknowledged that many critical points to which stochastic gradient descent (SGD) converges in high-dimensional spaces are saddle points \cite{chaudhari_entropy-sgd_2019}. Mathematically, saddle points lack a well-defined radius since there exist directions in which the loss decreases. Despite this, estimation techniques for the radius, such as the one proposed here, are commonly applied \cite{huang_understanding_2020}. Our convergence analysis supports the viability of our method, as demonstrated in Fig. \ref{appx_rad}a, where the estimation of the radius converges after approximately $3000$ randomly drawn samples. A possible explanation for this convergence could be the prevalence of positive directions. As illustrated in Table \ref{appx_table_eigendirections}, the Hessian has more positive eigendirections than zero or negative ones. This could lead to a much larger volume of parameters around the minima having an ascending loss. Further support for this hypothesis is provided in Fig. \ref{appx_rad}b, where random sampling of parameter vectors around minima, as described earlier, results more frequently in ascending than flat curves.

\begin{figure}[!htb]
  \centering \includegraphics[width=0.99\textwidth]{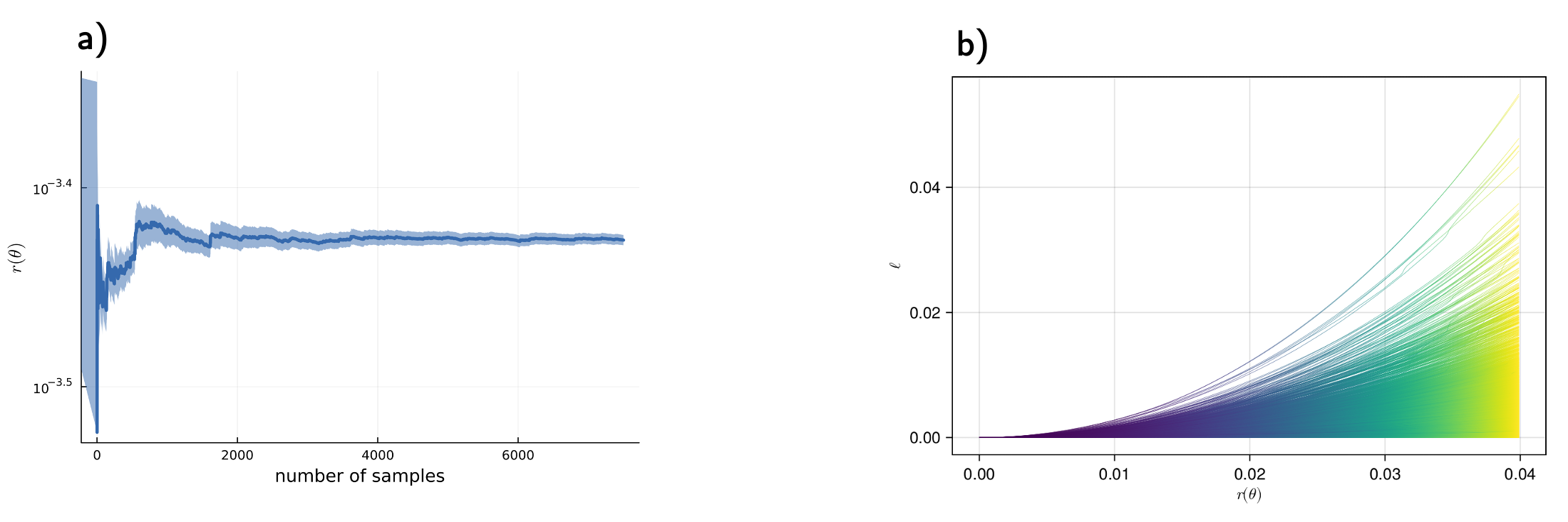}
  \caption{  \textbf{a)} Average radius as a function of the number of samples drawn from the loss landscape, indicating convergence to a constant radius at around $3000$ samples. \textbf{b)} Loss as a function of the radius for a shPLRNN $(N=2, M=100)$ trained on the Duffing system (Eq. \eqref{eq:duffing_vf}) and $5000$ sampled points. Curves are with kernel density smoothing.}
  \label{appx_rad}
  \normalsize
\end{figure}

\clearpage
\newpage

\section{Proofs}
\subsection{Proof of Theorem \ref{th_sensmulti}}
\label{app_proofthmulti}
We first clarify the mathematical interpretation of 'not reconstructed'. We consider two different scenarios, assuming the ground truth model $\Phi$ to be given (fixed):
\begin{itemize}
    \item[(i)] All trajectories starting in $B(A_k)$ converge to another reconstructed attractor $A_j, \ j \neq k$. This is often the case for trained models, as illustrated in Fig. \ref{fig_bench} and Fig. \ref{fig_bench_lorenz_multi}, and corresponds to a missing basin. In this case $\mathcal{E}^{M_{\mathrm{test}}}_{\mathrm{gen}}(\Phi_R) \propto \operatorname{vol}(B(A_k))$.
    \item[(ii)] There is an open set $B(A_{n+1}) \subset M$ corresponding to a new basin of a new attractor $A_{n+1}$ of $\Phi_R$, where we further assume $A_i \cap B(A_{n+1})= \emptyset , \quad \forall i \leq n$. This is the case when some additional dynamics is learned which is not present in the ground-truth system. In this case $\mathcal{E}^{M_{\mathrm{test}}}_{\mathrm{gen}}(\Phi_R) \propto \operatorname{vol}(B(A_{n+1}))$.
\end{itemize}

\paragraph{Proof of (i)}
We assume the decomposition as in Eq. \eqref{eq_statespacedecomp} for the ground-truth system $\Phi$:
\begin{align}
    M = \sqcup_{e=1}^N B_e \sqcup \Tilde{M} \quad \textnormal{such that} \quad \mu\big( \Tilde{M}\big) = 0
\end{align}
We will first prove the theorem for the \textbf{statistical error}. First, we want to show that the  error between the occupation measure of the ground-truth system and the reconstructed system on $B(A_k)$ is non-zero:
\begin{align}\label{eq_appx_eps}
 \textrm{SW}_1(\mu_{\bm{x}, T}^{\Phi}, \mu_{\bm{x}, T}^{\Phi_R}) \neq 0 \quad \forall \vx \in B(A_k).
\end{align}
To do so, we take some $\vx \in B(A_k)$. By assumption $\exists j\neq k: \omega(\vx,\Phi_R) \subseteq A_j$. We take some open subset  $U \supset A_j $ of $B(A_j)$, which has to exist as $B(A_j)$ is open. But for $\Phi$ it still holds that $ \omega(\vx,\Phi) \subseteq A_k$. It follows that
\begin{align}
     \mu_{\bm{x}, T}^{\Phi_R}(U) \neq \mu_{\bm{x}, T}^{\Phi}(U) ,
\end{align}
as, by definition of an attractor, there is some time $T'$ such that $\Phi_R(T',\vx)$ enters $U$, while $\Phi(t, \vx)$ never does for any $t$. Also note that $\mu_{\bm{x}, T}^{\Phi_R}(U)\neq 0$ as $\Phi_R$ enters $U$, making the occupation measure of $U$ non-zero. Consequently, we found a Borel set on which the two occupation measures disagree, hence  $ \mu_{\bm{x}, T}^{\Phi_R}\neq  \mu_{\bm{x}, T}^{\Phi}$. This construction can be repeated for any $\vx \in B(A_k)$.

Since $\textrm{SW}_1$  is a metric on the space of measures \cite{kolouri_sliced_2016}, this implies
\begin{align}
    \forall \vx \in B(A_k), \quad  \textrm{SW}_1(\mu_{\bm{x}, T}^{\Phi}, \mu_{\bm{x}, T}^{\Phi_R}) \neq 0 .
\end{align}
Since $B(A_k)$ is connected, also $\overline{B(A_k)}$ is connected. As $\overline{B(A_k)}$ is closed and $M$ compact, also $\overline{B(A_k)}$ is compact. By the mean-value theorem for integrals, there exists some $\vx' \in \overline{B(A_k)}$ such that
\begin{align}
     \mathcal{E}_{stat}^{B(A_k)} \big( \Phi_R  \big)= \textrm{SW}_1(\mu_{\bm{x'}, T}^{\Phi}, \mu_{\bm{x'}, T}^{\Phi_R}) \cdot  \operatorname{vol}(\overline{B(A_k)}).
\end{align}
 $\textrm{SW}_1(\mu_{\bm{x'}, T}^{\Phi}, \mu_{\bm{x'}, T}^{\Phi_R})$ is a constant, consequently $\mathcal{E}_{\mathrm{stat}}^{B(A_k)} \propto   \operatorname{vol}(\overline{B(A_k)})$ which proves the theorem.

We go on to prove the statement for the \textbf{topological error}. Denote by $D^n_r(\vy)$ a unit ball of radius $r$ in $\mathbb{R}^n$ centered on $\vy$.  Again, we take some $\vx \in B(A_k)$ and by assumption $\exists j\neq k: \omega(\vx,\Phi_R) \subseteq A_j$, while for $\Phi$ it holds that $ \omega(\vx,\Phi) \subseteq A_k$. As $M$ is equipped with the Euclidean distance, the Hausdorff distance between two sets measures the farthest possible Euclidean distance between a point in one set to the closest point in the other set. Denote by $S$ the set of points in $A_j$ being $\varepsilon_{d_H}$-close to $A_k$, 
\begin{align}
   S=\{\vy \in A_j| d_H(\{\vy\},A_k) \leq \varepsilon_{d_H} \} 
\end{align}
For reasonable choices of our hyperparameter $\varepsilon_{d_H}$ this set will be empty. Only if both attractors $A_j, A_k,$ lie very close to the boundary separating their basins of attraction, it could be non-zero. 
In this case we have 
\begin{align}
   d_H(\omega(\vx, \Phi_R ),  \, \omega(\vx, \Phi)) \leq \varepsilon_{d_H},  \quad \forall \vx \in \cup_{\vy \in S} D^n_{\varepsilon_{d_H}}(\vy)
\end{align}
and 
\begin{align}
    d_H(\omega(\vx, \Phi_R ),  \, \omega(\vx, \Phi)) > \varepsilon_{d_H}, \quad \forall \vx \in B(A_k)\backslash \cup_{\vy \in S}D^n_{\varepsilon_{d_H}}(\vy) .
\end{align}
Hence, the indicator function in Eq. \eqref{eq:etop} $\mathbbm{1}_{\Phi_R(\vx)}$ is $0$ on the set $B(A_k)\backslash\cup_{\vy \in S}D^n_{\varepsilon_{d_H}}(\vy)$, as the condition for the closeness of limit sets in the Hausdorff metric is violated.

Consequently,
\begin{align}
    \mathcal{E}^{M}_{\text{top}}(\Phi_R) = 1 - \frac{1}{\text{vol}(M)} \int_{ M} \mathbbm{1}_{\Phi_R}(\vx) \, d\vx =  \frac{\mathrm{vol}(B(A_k))-\mathrm{vol}(\cup_{\vy \in S}D^n_{\varepsilon_{d_H}}(\vy))}{\mathrm{vol}(M)} \propto   \frac{\mathrm{vol}(B(A_k))}{\mathrm{vol}(M)}
\end{align}
This result proves the theorem when choosing $M_{\mathrm{test}}=M$, given a choice of $\varepsilon_{d_H}$ that ensures that $\mathrm{vol}(B(A_k)) \gg \mathrm{vol}(\cup_{\vy \in S}D^n_{\varepsilon_{d_H}}(\vy))$. This should be generally fulfilled by a reasonably small choice of $\varepsilon_{d_H}$.\\

\paragraph{Proof of (ii)}
The proof is analogous to the one of (i) if we replace $A_k$ with $A_{n+1}$, but the main steps will be stated again for completeness. For the \textbf{statistical error} we want to prove a result similar to Eq. \eqref{eq_appx_eps}:
\begin{align}
 \textrm{SW}_1(\mu_{\bm{x}, T}^{\Phi}, \mu_{\bm{x}, T}^{\Phi_R})\neq 0 \quad \forall \vx \in B(A_{n+1}).
\end{align}
Assume some $\vx \in B(A_{n+1})$. We know that $\nexists j\leq n: \omega(\vx,\Phi_R) \subseteq A_j$, since by assumption $A_i \cap B(A_{n+1})= \emptyset \quad \forall i \leq n$. That is, while $\omega(\vx,\Phi_R) \subset A_{n+1}$, for $\Phi$ we have $\exists j\leq n:\omega(\vx,\Phi) \subseteq A_j$. 
In accordance with the proof of (i), we can construct a set $U \supset A_{n+1} $ and conclude with a similar line of arguments as above that
\begin{align}
     \textrm{SW}_1(\mu_{\bm{x}, T}^{\Phi}, \mu_{\bm{x}, T}^{\Phi_R}) \neq 0 \quad  \forall \vx \in B(A_{n+1}).
\end{align}
Using the mean-value theorem for integrals we can conclude
\begin{align}
    \mathcal{E}_{\mathrm{stat}}^{B(A_{n+1})}(\Phi_R) \propto \mathrm{vol}(B(A_{n+1}))
\end{align}
proving the statement if we set $M_{\mathrm{test}}=B(A_{n+1})$.\\
For the \textbf{topological error} we take an $\vx \in B(A_{n+1})$. We know that $\omega(\vx,\Phi_R) \in A_{n+1}$ and $ \exists j\leq n:\omega(\vx,\Phi) \subseteq A_j$. In the construction of the set $S$, we need to be cautious as $B(A_{n+1})$ could have many neighbouring basins. Accordingly, we have to construct a seperate set $S_i$ for each attractor. We denote by $S_i$ the set of points in $A_i$, $ \varepsilon_{d_H}$-close to $A_{n+1}$, 
\begin{align}
   S_i=\{\vy \in A_i| d_H(\{\vy\},A_{n+1}) \leq \varepsilon_{d_H} \} , \quad i \leq n
\end{align}
Again, each $S_i$ will be empty for reasonable choices of $\varepsilon_{d_H}$ and only non-empty for systems where $A_i$ lies very close to $A_{n+1}$. Now, by an argument analogous to (i), we can conclude that
\begin{align}
    \mathcal{E}^{M}_{\text{top}}(\Phi_R)   =  \frac{\mathrm{vol}(B(A_{n+1}))-\mathrm{vol}(\cup_{i=1}^n\cup_{\vy \in S_i}D^n_{\varepsilon_{d_H}}(\vy))}{\mathrm{vol}(M)} \propto   \frac{\mathrm{vol}(B(A_{n+1}))}{\mathrm{vol}(M)}
\end{align}
proving the result when choosing $M_{\mathrm{test}}=M$, given a reasonable choice of $\varepsilon_{d_H}$ (as discussed in (i)). 

\subsection{Proof of Theorem \ref{th_sc1}}
\label{app_proof_sindy}

For the proof of Theorem \ref{th_sc1}, we need the following definitions:

\begin{definition}

Given a trajectory $\Gamma_{\vx_0}$, we define the graph of this trajectory by 
\begin{equation}
	\Omega_{\vx_0}=
 I^f_{\vx_0} \times \Gamma_{\vx_0}.
\end{equation}

\end{definition}

\begin{definition}
	We define the following map, mapping the hypothesis class $\mathcal{H}$, a set of initial conditions and a set of times to the graph of the solution:
\begin{align}\label{eq_sigma}
 \sigma \, : \, & \mathcal{H} \times \mathbb{R}^{n }\times \mathbb{R} \rightarrow \mathbb{R} \times M, \quad (f,\vx_0,I^f_{\vx_0}) \mapsto \Omega_{\vx_0}    
\end{align}
 %---------------------
\end{definition}

The proof of theorem \ref{th_sc1} is based on the following lemmata: 
\begin{itemize}
    \item[(i)] Using Lemma \ref{app_lemma1}, we will first show that the set of solutions to the parameter estimation in SINDy (or any library-based algorithm) which is based on a minimization problem, can be rewritten as the pre-image $\pi_1(\sigma^{-1}(\Omega_{\vx_0}))$. $\pi_1$ denotes the projection on the first argument of the pre-image.
    \item[(ii)] Lemma \ref{appx_lemma2} then establishes how the pre-image and the associated minimization problem can be rewritten as a linear (matrix) equation.
    \item[(iii)] Lemma \ref{lemma_3} shows under which conditions this equation will have a unique solution.
    \item[(iv)] Finally, Lemma \ref{lemma_4} is used to characterize the set $\mathcal{H}_0$ for $\mathcal{H}=\mathcal{B}_L$, which will be used to prove the theorem.
\end{itemize}

 \begin{lemma}
 \label{app_lemma1}
 	Let $\Gamma_{\vx_0}=\{\vx(t)| t \in I_{\vx_0}^f\}$ be the solution (with graph $\Omega_{\vx_0}$) of some ground-truth VF $f_{GT} \in \mathcal{B}_L$, with parameters $\vtheta_{GT}$, $f_{GT}=f(\vx; \vtheta_{GT})$. For linearly parameterized function spaces (see Eq. \eqref{eq_funcsp}) like $\mathcal{B}_L$ we can write the (first projection of the) pre-image of a single trajectory as
 \begin{align}\label{eq_preim}
\pi_1(\sigma^{-1}(\Omega_{\vx_0}))&= \left\{ f \in \mathcal{B}_L |\Delta_j= \int_{t_0}^{t_1}  \left(\vx_j(t)-\vx_j(0)-\int_{0}^{t}f_{j}(\vx(s);\vtheta)ds \right)^2 dt=0,  \quad j \in \{1,...,n\} \right\}\\
 &= \min_{\vtheta} \{f(\vx; \vtheta)| \int_{t_0}^{t_1}\|\dot{\vx}(t)-f(\vx;\vtheta) \|_{L^2}^2 dt \}\label{eq_preim2}
  \end{align}
Note that the second equation exactly corresponds to the minimization problem stated by SINDy (without regularization).
   \end{lemma}

\begin{proof}
By the definition of Eq. \eqref{eq_sigma}, it holds that
\begin{equation}
\pi_1(\sigma^{-1}(\Omega_{\vx_0}))= \left\{ f(\vx;\vtheta) \in \mathcal{B}_L \, | \, \dot{\vx}(t)=f(\vx(t);\vtheta), \ \forall t \in I^f_{\vx_0}\quad \& \quad \vx(0)=\vx_0 \right\}.
\end{equation}
The following equivalence holds for all $j \in \{1,...,n\}$:
\begin{align}
&\quad \quad \quad \dot{x}_j(t)=f_{j}(\vx(t);\vtheta), \quad x_j(0)=\vx_{0,j}\\
 &\Leftrightarrow    \quad  x_j(t)-x_j(0)=\int_{0}^{t}f_{j}(\vx(s);\vtheta)ds \label{preim} \\ 
     &\Leftrightarrow    \quad \int_{t_0}^{t_1}  \left(x_j(t)-x_j(0)-\int_{0}^{t}f_{j}(\vx(s);\vtheta)ds \right)^2 dt=0   \label{appx_eq_ss}
\end{align}
	where we used in the last step that for a continuous function $f$ with $f(x) \geq 0, \forall x \in I \subseteq \mathbb{R}$ we have $\int f(x) dx=0 \quad \Rightarrow \quad f(x)=0$. This establishes Eq. \eqref{eq_preim}.\\ 
  For Eq. \eqref{eq_preim2}, let us assume $f \in  \pi_1(\sigma^{-1}(\Omega_{\vx_0}))$. By the definition of $\sigma$, we note that $f$ satisfies  for any $j \leq n$
  \begin{align}
      &\dot{x}_j(t)=f_{j}(\vx(t);\vtheta), \quad x_j(0)=\vx_{0,j}\\
      \Leftrightarrow \quad  &\dot{x}_j(t)-f_{j}(\vx(t);\vtheta)=0, \quad x_j(0)=\vx_{0,j}.
  \end{align}
  
By the same argument we used to get from Eq. \eqref{preim} to Eq. \eqref{appx_eq_ss}, it holds that $f$ fulfills
  \begin{align}
      \int_{t_0}^{t_1}\|\dot{\vx}(t)-f(\vx;\vtheta) \|_{L^2}^2 dt = 0 .
  \end{align}
 As the $L^2$-norm (its square) $\|\cdot\|_{L^2}^2$ is a positive function, the minimum value it can attain is $0$.  Consequently, $f \in  \min_{\vtheta} \{f(\vx; \vtheta)| \int_{t_0}^{t_1}\|\dot{\vx}(t)-f(\vx;\vtheta) \|_{L^2}^2dt\} $.
   For the other direction of the set-inclusion, assume $f \in  \min_{\vtheta} \{f(\vx; \vtheta)| \int_{t_0}^{t_1}\|\dot{\vx}(t)-f(\vx;\vtheta) \|_{L^2}^2 dt\} $. In general, $f$ could be different from $f_{GT}$ as there could exist many different $f$ solving the minimization problem. However, for $f_{GT}$ we know that the following holds:
   \begin{align}
        \int_{t_0}^{t_1}\|\dot{\vx}(t)-f_{GT}(\vx;\vtheta_{GT}) \|_{L^2}^2dt = 0.
   \end{align}
   
   It follows that $\forall f \in \min_{\vtheta} \{f(\vx; \vtheta)| \int_{t_0}^{t_1}\|\dot{\vx}(t)-f(\vx;\vtheta) \|_2^2dt\}$ we must have 
   \begin{align}
        \int_{t_0}^{t_1}\|\dot{\vx}(t)-f(\vx;\vtheta) \|_{L^2}^2dt = 0 ,
   \end{align}
   as the minimum value for $\int_{t_0}^{t_1}\|\dot{\vx}(t)-f(\vx;\vtheta) \|_{L^2}^2 dt$  is zero when choosing $\vtheta=\vtheta_{GT}$.
This establishes that any $f \in \min_{\vtheta} \{f(\vx; \vtheta)| \int_{t_0}^{t_1}\|\dot{\vx}(t)-f(\vx;\vtheta) \|_{L^2}^2dt\}$ has $\vx(t)$ as a solution. By definition of $\sigma$, we conclude that $f \in  \pi_1(\sigma^{-1}(\Omega_{\vx_0}))$.
   
\end{proof}
Based on Lemma \ref{app_lemma1}, studying the solution of the (unregularized) SINDy minimization problem, Eq. \eqref{eq_preim2}, comes down to studying the pre-image $\pi_1(\sigma^{-1}(\Omega_{\vx_0}))$. Further, note that there is a bijection between parameters $\vtheta \in \mathbb{R}^{m \times n}$ and functions $f(\vx; \vtheta)$ for a linearly parameterized function space such as $\mathcal{B}_L$. In the following, we will identify $\vtheta$ with $f(\vx; \vtheta)$, allowing us to write proofs for $f$ in terms of proofs for $\vtheta$.\\
In \citet{article}, the solutions to Eq. \eqref{eq_preim2} were studied for a one-dimensional setting. For higher dimensional settings, this takes the following form: 
\begin{lemma}\label{appx_lemma2}
$f(\vx; \vtheta)$ is an element of $ \pi_1(\sigma^{-1}(\Omega_{\vx_0}))$ 
if and only if $\vtheta$ solves the following matrix equation 
\begin{equation} \label{eq_mal}
	\underbrace{\begin{bmatrix} \left< \Psi _{ 0 }\cdot \Psi_0  \right>  & \left< \Psi _{ 0 }\cdot \Psi_1  \right>  & \cdots  & \left< \Psi _{ 0 }\cdot \Psi_N  \right>  \\ \left< \Psi _{ 1 }\cdot \Psi_0  \right>  & \left< \Psi _{ 1 }\cdot \Psi_1  \right>  &  & \left< \Psi _{ 1 }\cdot \Psi_N  \right>  \\ \vdots  & \vdots  & \ddots  & \vdots  \\ \left< \Psi _{ N }\cdot \Psi_0  \right>  & \left< \Psi _{ N }\cdot \Psi_1  \right>  & \cdots  & \left< \Psi _{ N }\cdot \Psi_N  \right>  \end{bmatrix}}_{\Bar{\Psi}}\underbrace{\begin{bmatrix} \theta_{ 1,j } \\  \theta_{ 2,j } \\ \vdots  \\ \theta_{ N,j } \end{bmatrix}}_{\vtheta_j}=\underbrace{\begin{bmatrix} \left< \Bar{\vx}_j \Psi _{ 0 } \right>  \\ \left< \Bar{\vx} _{ j }\Psi _{ 1 } \right>  \\ \vdots  \\ \left< \Bar{\vx} _{ j }\Psi _{ N } \right>  \end{bmatrix}}_{X_j} \quad \forall j=1,\dots,n
	\end{equation}
 with $\Bar{\vx}_j=\vx_j(t)-\vx(0)$ and 
\begin{align}
  \Psi_k  &=\int_0^t \psi_k(\vx(s))ds  \\
    \left< \Psi_k \cdot \Psi_i\right>&=\int_{t_0}^{t_1} \Psi_k(t) \cdot \Psi_i(t) dt= \int_{t_0}^{t_1}\int_{0}^{t}  \psi_k(\vx(s))ds \cdot \int_{0}^{t}  \psi_i(\vx(s))ds dt .
\end{align}
\end{lemma}

\begin{proof}
Consider the stationarity conditions 
\begin{equation*}
	   \frac{\partial \Delta_j}{\partial (\vx_{0})_j}=0 \quad \& \quad \frac{\partial \Delta_j}{\partial \theta_{k,j}}=0\quad \forall k =1,\dots,m \quad \forall j=1,...,n .
\end{equation*}
The second condition leads to
\begin{align}
	 \frac{\partial \Delta_j}{\partial \theta_{k,j}}&= \int_{t_0}^{t_1}  \left( 2\left( \vx_j(t)-\vx_j(0)-\int_{0}^{t}f_{j}(\vx(s);\vtheta)ds\right)   \cdot ( \frac{\partial }{\partial \theta_{k,j}} \int_{0}^{t}f_{j}(\vx(s);\vtheta)ds ) \right)dt \\
     &=\int_{t_0}^{t_{1}} 2\left( \left(\vx_{j}(t)-\vx_{j}(0)\right) \cdot \int^{t}_{0} \psi_{k}(\vx(s))ds-\left( \sum^{m}_{i=1}\theta_{i,j} \int_{0}^{t} \psi_{i}(\vx(s)) ds \right)\cdot \int^{t}_{0} \psi_{k}(\vx(s))ds \right) dt=0.
\end{align}
Using the linearity of the integral this can be rewritten as
\begin{equation}
    \sum^{m}_{i=1}\theta_{i,j} \int_{t_0}^{t_{1}} \left( \int_{0}^{t} \psi_{i}(\vx(s)) ds \cdot \int^{t}_{0} \psi_{k}(\vx(s))ds\right) dt= \int_{t_0}^{t_{1}} \left( \left(\vx_{j}(t)-\vx_{j}(0)\right) \cdot \int^{t}_{0} \psi_{k}(\vx(s))ds\right)dt
\end{equation}
leading to Eq. \eqref{eq_mal}.
\end{proof}

\begin{lemma}
\label{lemma_3}
 Let $\Omega_{\vx_0}$ be the graph of a trajectory from a VF $f \in \mathcal{B}_L$ with $\dot{\vx}(t)=f(\vx(t))$. It holds 
\begin{align}  \label{eq_equo}
\pi_1(\sigma^{-1}(\Omega_{\vx_0})) \textnormal{ is unique } \Longleftrightarrow \nexists \vtheta \in \mathbb{R}^{m }: \ \Lambda(\vx(t))=\sum_{i=1}^{m}\theta_{i}\psi_{i}(\vx(t))=0 \quad \forall t \in I_{\vx_0}^f .
\end{align}
That is, the given trajectory does not solve an algebraic equation in the basis functions. By `unique' we mean that $\pi_1(\sigma^{-1}(\Omega_{\vx_0})) $ contains only one element. 
\end{lemma}

\begin{proof}
    We directly prove the equivalence stated in Eq. \eqref{eq_equo}. Assuming $\pi_1(\sigma^{-1}(\Omega_{\vx_0}))$ is not unique, by lemma \ref{appx_lemma2} this is equivalent to matrix $\Bar{\Psi}$ being singular. This implies $\ker(\Bar{\Psi}) \neq \{0\}$, which in turn is equivalent to 
	$$\exists \vtheta \in \mathbb{R}^m \backslash \{0\}: \, \Bar{\Psi} \cdot \vtheta=0$$
 
	Hence, each row of the vector $ \Bar{\Psi} \cdot \vtheta$ needs to be zero, implying that \footnote{Note that in this proof we deviate from our standard notation in that  $\vtheta$ does not refer to the full matrix $\vtheta \in \mathbb{R}^{m \times  n}$ as in  Eq. \eqref{eq_funcsp}, but just to a vector in $\mathbb{R}^m$.}
	\begin{equation}\label{eq1}
	\forall k  \leq m \quad: \int_{t_0}^{t_1}\left[  \int_0^t \psi_k(\vx(s)) ds \cdot \int_0^t \sum_{i=1}^{m}\theta_{i}\psi_{i}(\vx(s))ds \right] dt=0
	\end{equation} 	

We need to show that the integral in Eq. \eqref{eq1} can only become zero iff $\sum_{i=1}^{m}\theta_{i}\psi_i(\vx(t))=0 \quad \forall t \in [t_0,t_1]$. By linearity of the integral over $t$, we can choose an arbitrary vector $\bm{\eta} \in \mathbb{R}^m$ and sum its components such that Eq. \eqref{eq1} is equal to the following:

\begin{equation}
	 \int_{t_0}^{t_1} \left[ \sum_{k=1}^m \eta_{k} \int_0^t \psi_k(\vx(s)) ds \cdot \int_0^t \sum_{i=1}^{m}\theta_{i}\psi_{i}(\vx(s))ds \right] dt=0
	\end{equation} 	
We are free to choose $\bm{\eta} =\vtheta$, leading to
\begin{equation}
	 \int_{t_0}^{t_1} \left[ \int_0^t \sum_{i=1}^{m}\vtheta_{i}\psi_{i}(\vx(s))ds \right]^2 dt=0
\end{equation} 
By the same argument as above, as the integral of a continuous, non-negative function can only be zero if the function itself is zero, we have
\begin{equation}
    \int_0^t \sum_{i=1}^{m}\vtheta_{i}\psi_{i}(\vx(s))ds=0 \quad \forall s \in [t_0,t_1] .
\end{equation}
Since this has to be true $\forall t \in [t_0,t_1]$, we finally conclude 
\begin{equation}
      \sum_{i=1}^{m}\vtheta_{i}\psi_{i}(\vx(t))=0 \quad \forall t \in [t_0,t_1]
\end{equation}

\end{proof}
\begin{lemma}
    \label{lemma_4}
   Consider two trajectories, $\Gamma_{\vy_0}=\{\vy(t)| t \in T, \vy(0) = \vy_0\}$ and $\Gamma_{\vx_0}=\{\vx(t)| t \in T,  \vx(0) = \vx_0\}$, with initial conditions $\vx_0$ and $\vy_0$ such that $\vx_0=\vy_0$. We further assume these trajectories arise from two evolution operators associated to the VFs $f, g$,  $\Phi^f, \Phi^g: \mathbb{R} \times M \rightarrow M$ with $\vx(t)=\Phi^f(t,\vx_0)$ and $\vy(t)=\Phi^g(t,\vy_0)$.  If, for all times $T \in \mathbb{R}$ and Borel sets $B \subseteq M$, the occupation measures associated with these trajectories satisfy

\begin{equation}\label{eq_meseq}
\mu_{\vx_{0}, T}(B) = \mu_{\vy_{0}, T}(B) \quad \forall B, \forall T \in \mathbb{R},
\end{equation}
then the underlying trajectories are identical:
\begin{equation}
\Gamma_{\vx_0} = \Gamma_{\vy_0}.
\end{equation}

\end{lemma}

\begin{proof}
The proof of this Lemma \ref{lemma_4} was adapted from \citet{dawkins}. 
As $M$ is compact and Hausdorff and the Lebesgue measure is regular, Eq. \eqref{eq_meseq} is equal to a different condition. Let $\varphi$ be a bounded continuous function $\varphi: M \rightarrow \mathbb{R}$, where we denote the set of all bounded continuous functions by $\mathcal{C}_0(M)$. Then, $\mu_{\vx_{0}, T} = \mu_{\vy_{0}, T}$ implies
\begin{align}
\int_0^T \varphi(\vx(s))\,ds = \int_0^T \varphi(\vy(s))\,ds,\qquad\forall T\ge 0, \ \forall \varphi \in \mathcal{C}_0(M)
\end{align}
where $\vx(t)$ and $\vy(t)$ are the trajectories corresponding to $\Phi^f$ and $\Phi^g$. Differentiating with respect to $T$ yields:
\begin{align}
\varphi(\vx(T)) =   \varphi(\vy(T)),\qquad\forall T\ge 0.
\end{align}
Bounded continuous functions separate points, which means that $\forall \vx_1\neq \vx_2 \in M $ and $\exists \varphi \in\mathcal{C}_0(M): \ \varphi(\vx_1) \neq \varphi(\vx_2)$. It follows that
\begin{align}
\vx(T) =   \vy(T),\qquad\forall T\ge 0.
\end{align}

\end{proof}

\paragraph{Proof of Theorem \ref{th_sc1}}
\label{app_th_sindy}
Based on these four lemmas, we can now prove Theorem \ref{th_sc1}.

\begin{proof}
For the sake of simplicity, we focus on a bistable system where $M_{\mathrm{train}}=B(A_1)$ and $M_{\mathrm{test}}=B(A_2)$, but the generalization to $n$ basins immediately follows. Let us denote the ground truth VF by $f_{GT} \in \mathcal{B}_L$ and write $f_{GT}(\vx; \vtheta_{GT})$ to make the parameter dependence explicit.  
    
Now, consider an arbitrary element $g$ in $\mathcal{B}_{L,0}=\{g \in \mathcal{B}_L| \mathcal{E}^{M_{\mathrm{train}}}_{\mathrm{gen}}(\Phi_g)=0\}$, where $\Phi_g$ denotes the evolution operator associated with $g$. $g \in \mathcal{B}_{L,0}$ implies $\mathcal{E}_{\mathrm{stat}}^{B(A_1)}(\Phi_g)=0$. It holds that 
\begin{align}
    \mathcal{E}_{\mathrm{stat}}^{B(A_1)}(\Phi_g)=0 \ \Longleftrightarrow \ \mu_{\vx_0,T}^{\Phi_f} = \mu_{\vx_0,T}^{\Phi_g} ,
\end{align}
based on $\textrm{SW}_1$ being a metric on the space of measures \cite{kolouri_sliced_2016}. This equality holds for all $T  \in I_{\vx_0}^f$  and for all $\vx_0 \in B(A_1)$. According to Lemma \ref{lemma_4}, this implies that all trajectories in $B(A_1)$ of $f_{GT}$ and $g$ coincide. Further, applying Lemma \ref{lemma_3} and considering the assumption that at least one trajectory does not satisfy an algebraic equation, which we will denote by $\Gamma'_{\vx_0'}$ (with graph $\Omega'_{\vx_0'}$), we can uniquely determine the parameters $\vtheta$ of $g(\vx; \vtheta) \in \mathcal{B}_L$ via $\pi_1(\sigma^{-1}(\Omega'_{\vx'_0}))$. Since $g$ and $f_{GT}$ share the same non-algebraic trajectory, $f$ and $g$ must be identical as $\pi_1(\sigma^{-1}(\Omega_{\vx_0}))$ only contains one element (Lemma \ref{lemma_3}). Consequently,
\begin{align}
    \mathcal{B}_{L,0}=\{f_{GT}\} .
\end{align}
Thus, $\mathcal{B}_{L,0}$ solely consists of $f_{GT}$. Since $f_{GT}$ has zero generalization error on $M$, it follows that any model in $\mathcal{B}_{L,0}$ exhibits zero generalization error on $M$. Therefore, $(\mathcal{B}_L, \mathcal{D})$ is strictly learnable.
\end{proof}

\begin{corollary}
\label{th_sc1_2}
When the observed trajectory $\Gamma_{\vx_0} \subset \mathcal{D}$ solves an algebraic equation, we can nevertheless restrict $\mathcal{B}_L$ to $\mathcal{B}_L' \subset \mathcal{B}_L $ in a way that $(\mathcal{B}_L',\mathcal{D})$ is strictly learnable. 
\end{corollary}
\begin{proof}
If $\Gamma_{\vx_0}$ solves an algebraic equation, this means that $\operatorname{ker}(\Bar{\Psi})\neq \{0\}$ (see Lemma \ref{lemma_3}). Hence, Eq. \eqref{eq_mal} 
\begin{align}
    \Bar{\Psi} \cdot\vtheta_j=X_j
\end{align}
does not have a unique solution for $\vtheta_j \neq 0$. However, for linear function spaces, we can use the fundamental theorem of homomorphisms to make $\Bar{\Psi}\backslash \ker(\Bar{\Psi})\rightarrow \operatorname{im} \Bar{\Psi}$ an isomorphism. Then, the system
\begin{align}
    \Bar{\Psi}' \cdot \vtheta_j'=X_j'
\end{align}
has a unique solution, where $\Bar{\Psi}' $ corresponds to a matrix where some basis functions $\psi_{i,j}$ are removed. In the corresponding coefficient vector $\vtheta'_j \in \mathbb{R}^{m'}$ the coefficients for the basis $\psi_{i,j}$ are removed as well, hence $m'< m$. By this process of removing basis functions $\psi_{i,j}$ from $\Bar{\Psi}$, thus restricting $\mathcal{B}_L$ to $B_L'$, we can make the solution unique again.
\end{proof}

\subsection{Proof of Theorem \ref{th_1_main}}
\label{app_proof_th1}

\begin{proof}
    Without loss of generality, we assume we have a bistable system, with $f$ the ground truth VF, $M_{\mathrm{train}}=B(A_1)$ and $M_{\mathrm{test}}=B(A_2)$. We aim to construct an infinite family of functions $G=\{g_{\alpha}|\alpha \in I\}$ with zero reconstruction error on $B(A_1)$ but non-zero error on $B(A_2)$. We denote by $\Phi_{g_{\alpha}}$ the evolution operator associated with $g_{\alpha}$. Let $V'=B(A_2) \backslash A_2$. Due to the assumption that $f$ is not topologically transitive on $B(A_2)$, we can choose some non-empty, open subset $V \subseteq V'$. As $M$ is locally Hausdorff, we can define a compact subset $K \subset V$ such that $K$ itself contains an open set. As $V$ is open, we define a bump function $\Lambda \in \mathcal{C}^{\infty}(M)$ \cite{lee_introduction_2012} that is zero outside of $V$ and equals $1$ on $K$.

The $i$-th component of a VF in $G$ is then defined as 
\begin{align}
    g_{\alpha,i}(\vx)= f_i(\vx)+ \Lambda(\vx)\cdot \left(- f_i(\vx) +s_{\alpha,i}(\vx)\right), \quad i=1,...,n ,
\end{align}
where $\vx_{\alpha} \in K$. As $K$ contains an open set, there are infinitely many different $\vx_{\alpha}$. By construction, $\mathcal{E}_{\mathrm{gen}}^{B(A_1)}(\Phi_{g_{\alpha}})=0$, since $\Lambda$ is zero outside of $V$, thus zero on $B(A_1)$. On $B(A_1)$, $g_{\alpha}$ has the form $ g_{\alpha,i}(\vx)= f_i(\vx)$ ($f$ has zero generalization error as it is the ground-truth VF). In contrast, on $K$, the VFs in $G$ reduce to
\begin{align}
    g_{\alpha}(\vx)=s_{\alpha}(\vx) ,
\end{align}
where we assume $s_{\alpha}$ to be a differentiable VF with an attracting equilibrium point at $\vx_{\alpha}$. This construction is feasible on any compact set $K$ in any dimension. Given that $\Lambda$ is smooth and $s_{\alpha}$ is differentiable, $g_{\alpha}$ remains a member of $\mathcal{X}^1$. 

Thus, $G$ forms an infinite family of VFs, each possessing an attracting equilibrium point at distinct positions in $K$. We will denote by $B(A_{\vx_{\alpha}}) \subseteq V$ the basin of attraction of the equilibrium point $\vx_{\alpha}$. All functions $g_{\alpha}$ have an attractor on $B(A_2)$ different from the one of the ground truth VF $f$. Additionally, $B(A_{\vx_{\alpha}}) \cap A_2 = \emptyset$ by construction of $V$ for any $\alpha$. Thus, we constructed the same setting encountered in the proof of Theorem \ref{th_sensmulti}, case (ii). Using the proof and setting $B(A_{n+1})$ (in the notation of Theorem \ref{th_sensmulti}, case (ii)) to $B(A_{\vx_{\alpha}})$, we can conclude that 
\begin{align}
 \mathcal{E}_{\mathrm{gen}}^{B(A_{\vx_{\alpha}})}(\Phi_{g_{\alpha}}) \geq \varepsilon
\end{align}
with $\varepsilon>0$. Since $V \subset B(A_2)$, it follows that $\mathcal{E}_{\mathrm{gen}}^{B(A_2)}(\Phi_{g_{\alpha}})\neq0$ for all $\alpha \in I$. 
Consequently, $(\mathcal{X}^1,\mathcal{D})$ is not strictly learnable, since we can construct infinitely many functions in $\mathcal{X}^1$ with non-zero generalization error \footnote{For the proof that $(\mathcal{X}^1,\mathcal{D})$ is not strictly learnable, a single VF as constructed above would have sufficed.}.
\end{proof}

\clearpage

\begin{figure}[H]
  \centering \includegraphics[width=1.0\textwidth]{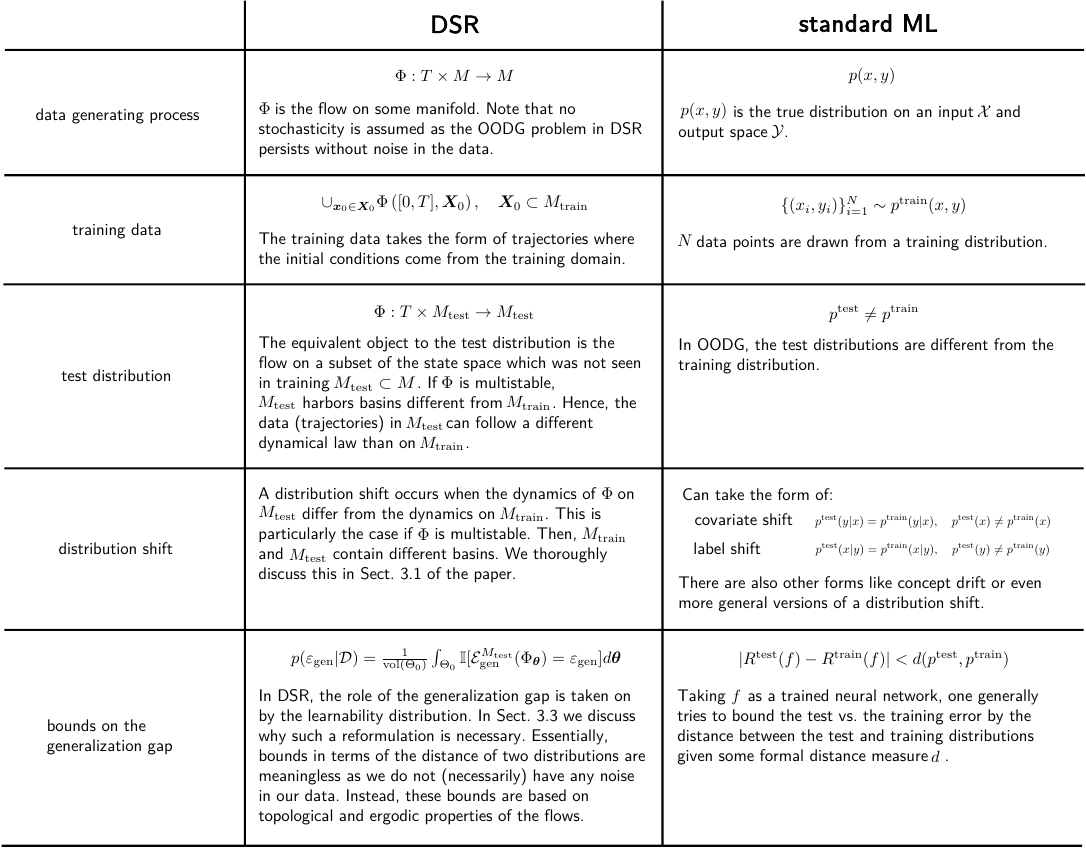}
  \caption{Comparison of OODG in DSR and standard ML.} 
  \label{table_oodg_comparison}
\end{figure}

\newpage

\bibliographybench{benchmark_systems}
\bibliographystylebench{icml2024}

\end{document}